\documentclass{article}
\usepackage{times}
\usepackage[T1]{fontenc}
\usepackage[utf8]{inputenc}
\usepackage{amsmath,amssymb,amsthm,mathtools}
\usepackage{graphicx,booktabs,tabularx,array}
\usepackage[export]{adjustbox}
\usepackage{xspace}
\usepackage{placeins}
\usepackage{algorithm,algorithmic}
\usepackage{microtype}
\usepackage{etoolbox}
\usepackage{url}            
\usepackage{amsfonts}       
\usepackage{nicefrac}       
\usepackage{comment}
\usepackage{mathtools}
\usepackage{amsmath}
\usepackage{array,makecell}
\usepackage{tabularx}
\usepackage{bbm}
\usepackage[toc,page]{appendix}
\usepackage{caption}
\usepackage{setspace}
\usepackage{xcolor}         
\usepackage{authblk}
\usepackage{listings}
\usepackage{wrapfig}
\usepackage{algorithm}
\usepackage{algorithmic}
\usepackage{float}
\usepackage{xspace}

\usepackage[hidelinks]{hyperref}
\newtheorem{theorem}{Theorem}
\newtheorem{lemma}{Lemma}
\newtheorem{proposition}{Proposition}
\newtheorem{corollary}{Corollary}
\theoremstyle{definition}

\newcommand{\E}{\mathbb E}
\newcommand{\Prob}{\mathbb P}
\newcommand{\ind}[1]{\mathbf 1\{#1\}}
\newcommand{\fhat}{\widehat f}
\newcommand{\muhat}{\widehat\mu}
\newcommand{\calT}{\mathcal T}
\newcommand{\calF}{\mathcal F}
\newcommand{\refdata}{\mathcal R}
\newcommand{\prior}{\tau_{\mathrm{prior}}}
\newcommand{\clip}{\operatorname{clip}}
\newcommand{\HARP}{\texttt{HARP}\xspace}
\newcommand{\DAPRO}{\texttt{DAPRO}\xspace}

\usepackage[numbers,sort&compress, square]{natbib}
\usepackage[margin=1.1in]{geometry}
\newcommand{\Dcal}{{\mathcal{D}_\mathrm{cal}}}

\newcommand{\calA}{\mathcal{A}}

\newcommand{\lpbvariant}{\HARP-D (+depth)}
\newcommand{\upbvariant}{\HARP-D (+depth)}
\newcommand{\metricvariant}{\HARP-D (no depth)}

\IfFileExists{manuscript_assets/selected_variants.tex}{\input{manuscript_assets/selected_variants.tex}}{}
\newcommand{\lpbchoice}{\lpbvariant; chosen once for the LPB task by the task-level criterion in Appendix~\ref{app:selection}.}

\newcommand{\metricchoice}{\metricvariant; chosen once for metric estimation by the task-level criterion in Appendix~\ref{app:selection}.}
\newcommand{\fullmethodname}{\textbf{H}ard-budget \textbf{A}llocation with \textbf{R}eflow for \textbf{P}redictive calibration}
\newcommand{\paperfigure}[2]{%
\includegraphics[width=\linewidth]{manuscript_assets/figures/#1.pdf}}
\title{Dynamic Budget Allocation for LLM Evaluation under Hard Resource Constraints}

\author[1]{Shai Feldman}
\author[1,2]{Yaniv Romano}
\affil[1]{Department of Computer Science, Technion IIT, Israel}
\affil[2]{Department of Electrical and Computer Engineering, Technion IIT, Israel}

\begin{document}
\maketitle
\begin{abstract}
We evaluate large language models (LLMs) in multi-turn interactions through their \emph{time-to-event}: the number of interaction steps required to produce an event of interest, such as a successful jailbreak or agentic task completion. Under limited compute, interactions may be terminated before the event occurs, so that event times are only partially observed (censored). 
Existing allocation methods for calibrating time-to-event bounds satisfy the budget only in expectation and can exceed the available budget on a particular evaluation run. 
Enforcing a hard constraint is particularly challenging as the cost of a trajectory is initially unknown.
We introduce \emph{\fullmethodname} (\HARP), a budget allocation that satisfies hard resource constraints and adaptively reallocates unused budget.
We show how to use \HARP to construct lower predictive bounds (LPBs) on the time-to-event and to estimate evaluation metrics such as the jailbreak rate on a fixed benchmark.
Although \HARP induces dependence in acquisition decisions across different trajectories, we prove that \HARP never exceeds the target budget, that its LPBs have finite-sample coverage guarantees, and that its metric estimates are unbiased.
Experiments on agentic task success, LLM jailbreaks, toxic content generation, and RAG hallucinations show that \HARP achieves coverage close to the nominal level with low variance, while never exceeding the given budget.
\end{abstract}

\section{Introduction}
\label{sec:intro}
Evaluating LLMs and agents through repeated interaction can require substantially more computation than assessing a single response. For example, a model may initially refuse a harmful request but comply after several revisions, or initially fail a coding task while succeeding with further assistance. Time to such an event describes behavior that a one-turn success rate does not capture. Lower and upper predictive bounds (LPBs and UPBs) on this time-to-event can support evaluating the safety and utility of agents and LLMs, but constructing reliable bounds requires observations from trajectories of unknown length.

In this work, we consider a \emph{hard} budget constraint for running LLM interactions, as required in many practical scenarios. A policy satisfying the budget in expectation may exceed it on specific runs. Reserving the maximal trajectory length for each sampled prompt may satisfy the deterministic constraint, but it can underutilize the budget when the event occurs early; conversely, acquiring until exhaustion and treating the observed sample as representative may induce a distribution shift which invalidates the theoretical guarantees. 

Designing this budget allocation under hard constraints raises three challenges. First, the budget cost of a trajectory is stochastic and unknown before acquisition. 
Second, the budget allocation must incorporate a dependence of the budget spent on different trajectories. Thus, the theoretical coverage guarantee must account for dependence between trajectories induced by the allocation. Prior works~\citep{gui2024conformalized,davidovconformalized, davidov2026calibrated, feldman2026spend} require that the censoring times are independent across different trajectories.
Third, samples in which an event occurs early save budget units, which we aim to reflow to other samples. This reallocation must preserve the feasibility of the remaining trajectories, so that the deterministic constraint holds throughout execution.

\citet{davidov2026calibrated} proposed a budget allocation and a calibration scheme for constructing calibrated LPBs with theoretical coverage guarantees.
\DAPRO~\citep{feldman2026spend} extended this approach to a dynamic budget allocation for multi-turn LLM evaluation that leverages the trajectory history to allocate the budget more efficiently. Specifically, it can dynamically choose to stop acquiring a trajectory at any turn, and the remaining budget will be exploited by other samples.
While both approaches provide informative predictive bounds and metric estimates with theoretical guarantees, they satisfy the budget constraint only in expectation rather than deterministically.

Beyond controlling the budget deterministically, we seek to derive low-variance estimators because evaluation is typically performed on a single realized dataset. Specifically, lower variability in the coverage rate of predictive bounds means that the coverage attained on a given dataset is more likely to be close to the nominal level \(1-\alpha\). For metric estimation, even if two estimators are unbiased, the lower-variance estimator is preferable as, on any single dataset, it is expected to yield an estimate closer to the true marginal value. Thus, the estimator with the lower variance provides more reliable evaluations.

The work of~\citet{pmlr-v286-yamin25a} studies budget allocation under hard constraints. In particular, they use dependent rounding~\citep{madow1944theory, deville2004efficient, gandhi2006dependent} for budget-constrained experimental design. Nonetheless, these works do not show how to construct calibrated predictive bounds with finite-sample distribution-free guarantees under hard constraints where the acquisition cost of each sample is arbitrary, stochastic, and initially unknown.

\subsection{Our contribution}

In this work, we introduce \emph{\fullmethodname} (\textbf{\HARP}), a novel budget allocation with deterministic guarantees to form a statistically valid bound for the time-to-event or a general evaluation metric. Our method has both static~\citep{davidov2026calibrated} and dynamic~\citep{feldman2026spend} variants. We outline our main contributions as follows:
\begin{itemize}
\item \textbf{Coverage and unbiasedness guarantees under dependent acquisition:}
We show that \HARP constructs predictive bounds with finite-sample coverage guarantees and produces unbiased metric estimates even though its acquisition decisions are dependent across samples.
Existing guarantees~\citep{gui2024conformalized,
davidov2026calibrated, feldman2026spend} assume that the acquisition is independent across samples to provide such guarantees. \HARP violates this assumption by design: decisions are coupled within groups. Our proof exploits the specific allocation structure of \HARP.
\item \textbf{Budget validity}: We formally prove that \HARP satisfies the deterministic budget constraint. Empirically, we observe that \HARP has a higher budget utilization compared to prior methods~\citep{feldman2026spend}.
\item \textbf{Depth and execution variants:} We develop a depth rule that, as observed empirically, leads to reduced coverage variance of the constructed calibrated predictive bounds and was not considered in prior works.
\item \textbf{Empirical Evaluation:} We demonstrate \HARP across diverse tasks, including agentic task success, toxic content, LLM jailbreaks, and RAG-based hallucinations. Our experiments show that our method satisfies the deterministic budget constraint and achieves low variance despite our method having a more restrictive deterministic resource constraint.

\end{itemize}

The software package implementing our method is available at \url{https://github.com/Shai128/harp}

\section{Background}
\label{sec:background}
\paragraph{Multi-turn evaluation.}
Let $X$ denote an initial prompt and $H(t)$ its trajectory history after $t$ interactions. 
We suppose we are given $N$ initial prompts $\Dcal=\{X_i\}_{i=1}^N$ as the calibration data.

We assume a model \(\calA\) interacts sequentially with a target model \(\mathcal{G}\) from an initial prompt \(X_i \in \Dcal\). The model \(\calA\)  aims to elicit a key event. The interaction begins with \(A_i(1)=X_i\). At each step \(t\), the target model produces a response \(R_i(t)=\mathcal{G}(A_i(t))\), which is then evaluated by an audit function \(\mathcal{J}\) to obtain an indicator for the event of interest
$
Y_i(t)=\mathcal{J}(A_i(t),R_i(t))\in\{0,1\},
$
where \(Y_i(t)=1\) indicates that the event has occurred. If \(Y_i(t)=1\), the interaction terminates. Otherwise, \(\calA\) uses the interaction history \(H_i(t)\), which contains previous prompts, responses, and audit annotations, to generate the next prompt,
\[
A_i(t+1)=\calA(H_i(t)).
\]
The process continues until the event occurs or the interaction is otherwise terminated.

The corresponding time-to-event is defined as $T_i = \min \{t \geq 1: Y_i(t) = 1\}$. To make this procedure tractable, we set the maximal sequence length to $M$.
For sample $i$, the censoring time $C_i\le M$ determines the maximal observed length $\min(T_i,C_i)$. If the event has not occurred by turn $C_i$, the trajectory is right-censored and its event time is unobserved. 

We use a pre-trained survival model that generates predictions $\fhat_\tau(X)$ for the $\tau$-th quantile of the time-to-event. 
Any model that estimates conditional quantiles can be used, e.g., quantile regression~\citep{QR, beyond_pinball_loss, feldman2021improving}.
The quantile predictions are nondecreasing in $\tau$, as quantiles of a single predicted distribution, and the candidate set $\calT$ contains a level $\tau_0$ with $\fhat_{\tau_0}\equiv0$.
We trim the predictions at $M$, so that $\fhat_\tau(X) \leq M$, to correspond to the maximal length we allow. We define
\begin{equation}
 A_i(\tau)=\ind{T_i<\fhat_\tau(X_i)},\qquad
 \mu(\tau)=\Prob\{T<\fhat_\tau(X)\}.                 \label{eq:loss}
\end{equation}
Given $\alpha,\delta\in(0,1)$, our goal is to construct a predictive bound $\fhat_{\widehat\tau}$ with miscoverage at most $\alpha$ up to a deviation $\Delta$ which arises from the uncertainty due to the finite calibration sample. Specifically, we seek the PAC guarantee
\begin{equation}
 \Prob\{\mu(\widehat\tau)>\alpha+\Delta\}\le\delta. \label{eq:pac}
\end{equation}
The outer probability is taken over calibration trajectories and acquisition randomization, while the inner one concerns an independent test point sampled from the same population. Thus, this guarantee differs from marginal coverage averaged over the multiple runs of the calibration procedure, common in classic conformal prediction~\citep{vovk2005algorithmic}. Throughout this work, the randomizations are taken conditionally on the pre-trained predictive model, when used, and an independent reference dataset.

\paragraph{Random acquisition and calibration.}
Let $Z_i$ indicate that sample $i$ is \emph{resolved}, i.e., that its observed prefix determines $A_i(\tau)$ for every candidate $\tau$. Let $\pi_i$ be the probability that $Z_i=1$, conditional on the complete trajectory of sample $i$ and on the history of the previously acquired groups (Appendix~\ref{app:filtration}).
The Horvitz--Thompson estimator \citep{Horvitz01121952} then gives
\begin{equation}
 \muhat(\tau)=\frac1N\sum_{i=1}^N\frac{Z_i A_i(\tau)}{\pi_i},\qquad
 \widehat\tau=\max\left\{\tau\in\calT:
       \max_{u\in\calT:u\le\tau}\muhat(u)\le\alpha\right\}.
 \label{eq:calibration}
\end{equation}
The denominator considers every sample in the calibration set, including prompts with unobserved event times. In practice, we only need access to the probability $\pi_i$ for samples with $Z_i A_i(\tau) = 1$.

\paragraph{From static allocation to \DAPRO.}
\citet{davidov2026calibrated} propose optimizing the sampling probabilities to minimize the variance of the empirical coverage rate. \DAPRO~\citep{feldman2026spend} extends this approach by updating continuation decisions dynamically from the trajectory history. It splits the calibration data into two: the first is used to learn an optimized allocation policy deployed on the second split. It further includes an internal budget calibration step based on conformal risk control (CRC)~\citep{angelopoulos2024conformal} to provide an expected budget guarantee.

For UPBs, the coverage error of the bound $\ind{T>\fhat_\tau(X)}$ decreases with $\tau$, and an infinite bound has zero error. Population metrics instead target an expectation, without selecting a quantile level. These changes require different acquisition depths and estimators. We therefore develop them separately in Appendices~\ref{app:upb} and~\ref{app:metrics}.

\section{Proposed Method and Theory}
\label{sec:method}
We split the calibration data into groups and allocate the budget iteratively over groups, based on the remaining budget each iteration. For each group, we set randomized censoring times so that the worst-case cost does not exceed the remaining budget, acquire the interactions based on them, and update the remaining budget according to the turns actually taken.

\subsection{Reservations, commitments, and reflow}
We rely on a reference set $\refdata$ of size $|\refdata|=R$ to compute estimated predictions, without spending budget on it.
We exclude the reference samples from the calibration estimator. In the experiments, the total cap is $B=N_0\bar B$, where $N_0 = N + R$ and $\bar B$ is the budget per parent prompt. The reference therefore saves interactions at the cost of a smaller estimation sample.

We randomly divide the calibration samples into $K=\lceil N/m\rceil$ groups $I_k$, each of size $m_k\le m$. We assign sample $i$ a maximal execution horizon $d_i\in\{0,\ldots,M\}$, determined by the reference design and the sample's own initial features. Let $L_{k-1}$ be the remaining budget before group $k$, initialized at $L_0=B$. We choose $\eta>0$ and require
\begin{equation}
 B\ge K(1+\eta)M.                                     \label{eq:feasible}
\end{equation}
The budget allowance and the pacing budget of the current group are
\begin{equation}
 \kappa_k=L_{k-1}-(K-k)(1+\eta)M,\qquad
 \beta_k=\frac{L_{k-1}}{K-k+1}.                         \label{eq:ledger}
\end{equation}
The allowance reserves $(1+\eta)M$ turns for each later group, while the pacing budget distributes the remaining budget over the remaining groups. Consequently, an inaccurate estimate of the consumed budget cost might not attain $\beta_k$, but it will not violate the allowance $\kappa_k$.

We denote by $\hat{r}_j(\tau)$ the estimated probability for miscoverage event $\ind{T_j<\fhat_\tau(X_j)}$. Next, we define
\begin{equation}
    \tau_\star := \min \arg \min_{\tau \in \calT} \left| \frac{1}{R} \sum_{j\in \refdata}\hat{r}_j(\tau) - \alpha \right|.
\end{equation}
The average predicted event probability is formulated as $\bar r := R^{-1}\sum_{j \in \refdata}\hat{r}_j(\tau_\star)$. 
We denote the estimated expected cost by $\widehat c_i=\widehat\E[\min(T_i,d_i)\mid X_i]$, and the average estimated cost function over the reference data is thus
\begin{equation}
    c_{\refdata}(\lambda) :=R^{-1}\sum_{j \in \refdata} \pi_{j}(\lambda)c_j \qquad
    \pi_j(\lambda)=\min\left(1, \max\left(\sqrt{\nu_j/(\lambda c_j)},p_{\min}\right) \right).
\end{equation}
We write $\nu_i=(1-\omega)+\omega\widehat r_i(\tau_\star)/\bar r$, with $\omega=1/2$. We choose the multiplier $\lambda_k$ to satisfy $c_{\refdata}(\lambda_k)\approx\beta_k/m_k$ and, for $d_i>0$, set
\begin{equation}
 p_i^{\rm prop}=\clip\!\left(\sqrt{\frac{\nu_i}{\lambda_k\widehat c_i}},p_{\min},1\right),
 \qquad p_i^c=\min\left\{p_i^{\rm prop},1,
                    \frac{\kappa_k-M}{m_kd_i}\right\}. \label{eq:commit}
\end{equation}
The specific formulation of $p_i^{\rm prop}$ is chosen as a heuristic, based on the closed-form minimizer derived by~\citet{davidov2026calibrated}. Our guarantees hold for any definition of $p_i^{\rm prop}$ that depends only on the sample's own features and previous groups.

We next apply weighted pivotal rounding~\citep{byrka2016improved} to sample a commitment indicator $\zeta_i$ with probability $p_i^c$.
The commitment indicator \(\zeta_i\) determines whether we commit to allocating budget to sample \(i\) to observe its event time \(T_i\); if \(\zeta_i=0\), we set \(C_i=0\). This pivotal rounding algorithm begins by initializing $\zeta_i = p_i^c$, and runs iteratively, where at each iteration, it randomly rounds at least one index $j$ into $\zeta_j \in \{0,1\}$, while preserving the identity  $\sum_{i} d_i \zeta_i$ throughout the process. The rounding procedure terminates once at most one coordinate remains unrounded. This last coordinate is then rounded with a Bernoulli draw that changes the weighted sum by at most $\max_i d_i \leq M$.
Combining this property with the cap in~\eqref{eq:commit} gives
\begin{equation}
 \E[\zeta_i\mid\text{group data, past}]=p_i^c,\qquad
 \sum_{i\in I_k}d_i\zeta_i\le\sum_{i\in I_k}d_ip_i^c+M\le\kappa_k.
 \label{eq:roundbudget}
\end{equation}
In Appendix~\ref{app:pivot} we outline the full algorithm we use and prove the properties of this algorithm. After executing the group, we subtract its actual paid cost $S_k$ from the ledger, $L_k=L_{k-1}-S_k$. Thus, turns saved by early events or randomized stopping are available to the next group.

\subsection{Completion, continuation, and depth}
\label{sec:variants}
\textbf{\HARP-Commit} (\textbf{\HARP-C}) runs each sample with $\zeta_i=1$ until resolving it, i.e., reaching its event time or horizon $d_i$, costing $\min(T_i,d_i)$ units for the $i$-th sample, and the resolution probability is $p_i^c$. This version adapts the budget \emph{between groups} through the ledger. \textbf{\HARP-Dynamic} (\textbf{\HARP-D}) additionally adapts the budget dynamically \emph{within a trajectory}. Before spending budget for turn $t$, it uses only $H_i(t-1)$ to compute a score $\nu_i(t)$, which is a smoothed miscoverage probability at step $t$, and estimates the remaining cost $\widehat e_i(t)$. 
We denote by $a_i(t)$ the probability of reaching time $t$, initialized at $p_i^c$. We set
\begin{equation}
 u_i(t)=\min\!\left\{1,\sqrt{\frac{\nu_i(t)}{\lambda_k\widehat e_i(t)}}\right\},\quad
 P_i(t)=\clip\left(\frac{u_i(t)}{a_i(t-1)},p_{\min},1\right),\quad
 a_i(t)=a_i(t-1)P_i(t).                               \label{eq:continue}
\end{equation}
So, at each step $t$, we continue to the next step with probability $P_i(t)$, and this random draw is independent across different samples and of the randomization of $\zeta_i$.
Importantly, we do not continue running a trajectory that was stopped by a continuation coin. 
For $b_i=\min(T_i,d_i)$, sample $i$ is resolved, i.e., $Z_i=\ind{C_i\ge b_i}$, when acquisition reaches its event time $T_i$ or its horizon $d_i$. Conditional on its trajectory and the previous groups, this occurs with probability
\begin{equation}
 \pi_i=\begin{cases}
 p_i^c, \quad \text{mode $C$},\\ 
 p_i^c\prod_{t=1}^{b_i}P_i(t), \quad \text{mode $D$}
 \end{cases}                  \label{eq:pi}
\end{equation}
These $Z_i$ and $\pi_i$ are used in the estimator~\eqref{eq:calibration}.

\paragraph{Depth.}

The vanilla version of \HARP sets the acquisition horizon as $d_i=\fhat_{\prior}(X_i)$, where $\prior$ is a prior quantile level chosen to reflect the belief that the calibrated level satisfies ($\widehat\tau\leq\prior$). Prior work provides examples of such choices: for a target coverage level of $1-\alpha=90\%$,~\citep{davidov2026calibrated, feldman2026spend} consider $(\prior\in\{0.3,0.56\})$.

The \emph{depth} version of \HARP chooses a quantile level $\ell\le\prior$ based on the initial reference predictions and restricts the candidate set to \(\tau\leq\ell\). It then sets the resolution horizon to $d_i=(\fhat_\ell(X_i)-1)_+$, which is sufficient to determine all retained events \(T_i<\widehat f_\tau(X_i)\). By avoiding acquisition at later interaction steps, we save budget at the cost of restricting the largest LPB that can be selected.

We choose the depth $\ell$ using only the initial predictions for the independent reference prompts $\refdata$. For each candidate $\tau$, these predictions estimate the average resolution cost, LPB miscoverage, and LPB size. We consider five risk scenarios, multiplying the predicted miscoverage curve by $\gamma\in\{.6,.8,1,1.25,1.6\}$ while keeping predicted costs and sizes fixed. For each $\tau$ and scenario $\gamma$, we estimate the largest mean LPB size among its retained candidates whose modeled risk plus a planning uncertainty penalty is at most $\alpha_{\rm design}$. 
The regret is defined as the signed difference of the LPB size compared to the best observed in the same configuration of risk $\gamma$.
We choose the level $\ell$ with the smallest worst-case regret, then increase it if a deeper endpoint still allows the planner's hypothetical acquisition fraction to remain one. Appendix~\ref{app:depth} provides the full derivations. 

Truncating the candidate family would change the UPB and metric estimation tasks. For these tasks, we instead independently assign each sample either a shallow or a full horizon, choosing the full-horizon branch with probability $.25$ and correcting for this draw in the weights; see Appendices~\ref{app:upb}--\ref{app:metrics} for more details. 

\begin{algorithm}[t]
\caption{Core \HARP acquisition and LPB calibration}
\label{alg:core}
\begin{algorithmic}[1]
\REQUIRE Pre-trained model, independent reference, calibration samples, parameters $B,M,m,\eta,p_{\min}$
\STATE Fix reference design and retained family; set $L_0=B$; form random groups.
\FOR{$k=1,\ldots,K$}
\STATE Compute reserve, allowance, and pacing target by~\eqref{eq:ledger}.
\STATE Choose $\lambda_k$ from the reference curve; cap commitments by~\eqref{eq:commit}.
\STATE Draw weighted pivotal commitments with horizons $d_i$ as costs (Algorithm~\ref{alg:pivot}).
\STATE Execute C or D; record paid turns, resolution indicators, and propensities~\eqref{eq:pi}.
\STATE Update $L_k=L_{k-1}-S_k$ using actual group spend.
\ENDFOR
\STATE Compute~\eqref{eq:calibration} over all $N$ working samples; return $\fhat_{\widehat\tau}$.
\end{algorithmic}
\end{algorithm}

\subsection{Theoretical coverage and budget guarantees}
\label{sec:guarantees}
We begin with the budget guarantee, which holds for every realization of the acquisition procedure.
\begin{theorem}[Deterministic budget control]\label{thm:budget}
Under~\eqref{eq:feasible}, \HARP satisfies $\sum_{k=1}^K S_k\le B$ almost surely.
\end{theorem}
Crucially, this guarantee holds even when the predictions are inaccurate or when the data is not i.i.d. The full proof is provided in Appendix~\ref{app:budget}.

For the proof, we define the filtration $\calF_{k-1}$ to contain the fixed design objects, all complete trajectories in previous groups, and all random coins used in those groups.
For a fixed candidate, define its predictable second-moment average
\begin{equation}
 v(\tau)=\frac1N\sum_{k=1}^K\sum_{i\in I_k}
       \E\left[\frac{A_i(\tau)}{\pi_i}\,\middle|\,\calF_{k-1}\right].
 \label{eq:varianceenvelope}
\end{equation}
The average is random but predictable. We assume an upper bound $v_\star$ that is deterministic conditional on the fixed reference set and holds almost surely for every retained candidate. This is a population envelope, rather than an estimated sample variance.
\begin{theorem}[LPB PAC bound]\label{thm:lpb}
Assuming the calibration and test data are i.i.d., suppose $v(\tau)\le v_\star$ almost surely for every candidate $\tau$. Set $L=\log(1/\delta)$. For any $\Delta\ge0$ satisfying
\begin{equation}
 \Delta\ge\frac{L}{3N}+\sqrt{\frac{L^2}{9N^2}
                       +\frac{2(v_\star-\alpha^2)_+L}{N}},
 \label{eq:lpbdelta}
\end{equation}
the LPB chosen by~\eqref{eq:calibration} satisfies~\eqref{eq:pac}.
\end{theorem}
With probability at least $1-\delta$, the selected LPB has population miscoverage at most $\alpha+\Delta$, with dependent acquisition entering the correction only through $v_\star$. See Appendix~\ref{app:lpbproof} for the proof.  
Notice that the coverage guarantee holds despite the dependence between the censoring times $C_i$. \HARP inherently violates independence across samples through both the pivotal selection and the ledger.

\paragraph{UPBs require a separate argument.}
The UPB estimator updates its predictions at every reached prefix:
\begin{equation}
 V_i(\tau)=m_{i0}(\tau)+\sum_t\frac{R_{it}}{\rho_{it}}
                        [m_{it}(\tau)-m_{i,t-1}(\tau)].
 \label{eq:aht}
\end{equation}
Here $m_{it}\in[0,1]$ predicts the finite-bound loss. We replace it with the true loss upon resolution and keep it fixed thereafter. The indicator $R_{it}$ denotes whether turn $t$ is reached, and $\rho_{it}$ accounts for commitment, continuation, and the insured-depth correction when used. Since the prediction increments can be signed, this estimator requires a different argument from the nonnegative LPB product bound.
\begin{theorem}[UPB PAC bound for signed updates]\label{thm:upb}
Assuming the calibration and test data are i.i.d., use the depth/propensity and prediction-path construction of Appendix~\ref{app:upb}. Suppose the effective reach weights $\rho_{it}$ are nonincreasing along each prefix (equivalently, their inverses are nondecreasing) and $\rho_{it}\ge1/W>0$ on every prefix that can contribute, with $W$ fixed conditional on the reference. Then the UPB suffix-envelope selector with an infinite fallback satisfies~\eqref{eq:pac} for UPB miscoverage with
\begin{equation}
 \Delta_{\rm U}=(2W-1)\sqrt{\frac1N\log\frac1\delta}.
 \label{eq:upbdelta}
\end{equation}
\end{theorem}
Acquisition enters this correction only through $W$. Summation by parts bounds each sample contribution in $[1-W,W]$, and conditioning on sample-local counterfactual continuations makes it affine in the commitment indicator. We apply the two pivotal cylinder inequalities to positive and negative coefficients separately, then use Cauchy--Schwarz to compare exponential moments with independent sample marginals at a factor-of-two cost. Combining these conditional bounds across groups handles reflow. Appendix~\ref{app:signedmgf} proves the result and a variance-sensitive version. Small propensities can still make the range-only bound vacuous.

\section{Experiments}
\label{sec:experiments}
We evaluate \HARP using the same datasets and setup as in~\citep{feldman2026spend}: toxicity, red teaming with Qwen and Llama Guard judges, retrieval-grounded hallucination, and AutoIF instruction following. Each setup uses Qwen 2.5 14B, Llama 3.1 8B, Phi-4 Mini, and Gemma 3 4B as target models. We use a survival model pre-trained on 4000 independent samples to predict event time through $M=200$. We remark that the budget constraint is for the calibration step only; the training phase has a separate budget.

\paragraph{Protocol and budget choice.}
We fix one target budget per dataset and task. For LPBs, we follow the configuration of~\citet{feldman2026spend} and set target $\bar B=20$ budget units per sample for the toxicity data and $\bar B=10$ for other datasets. For constructing UPBs and estimating metrics, we set $\bar B=20$ throughout.
Appendix~\ref{app:experiments} provides the full evaluation protocol.

Each experiment is repeated twenty times for random splits of calibration (2000) and test samples (1000). 
\HARP uses 100 reference prompts for which no budget is consumed, and the budget is allocated over the remaining 1,900. For \HARP, we set $m=25$, $\eta=.05$, and $p_{\min}=.02$. For \DAPRO, we use 25 data points for learning an optimized allocation policy and a different 25 samples for the CRC calibration, as in~\citep{feldman2026spend}. We compare \HARP-C and \HARP-D with and without depth, then choose one variant per task using coverage error and bound size for predictive bounds, and normalized estimation MSE for metrics (Appendix~\ref{app:selection}). The resulting choices are \lpbvariant\ for LPBs, \upbvariant\ for UPBs, and \metricvariant\ for metrics, applied to every dataset and target. Full comparison of the four variants is given in Appendix~\ref{app:ablation}.

\begin{figure}[t]\centering
\paperfigure{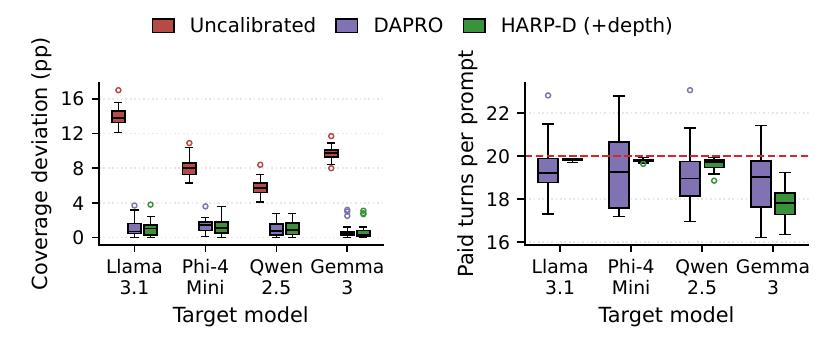}{2.3}
\caption{\textbf{Toxicity LPBs.} Coverage comparison: absolute coverage deviation and paid turns per parent prompt, all four target models, 90\% coverage, $\bar B=20$ turns per sample, twenty runs. \lpbchoice\ Uncalibrated uses raw quantiles and has no cost box. The dashed red line marks the budget.}
\label{fig:toxicity}\end{figure}

\paragraph{Predictive bounds.}
On toxicity (Figure~\ref{fig:toxicity}), \DAPRO\ and \HARP\ attain mean absolute coverage deviations of 1.08 and 1.05 percentage points (pp), averaged across four targets. For Qwen 2.5 on AutoIF (Figure~\ref{fig:autoif}), \DAPRO\ has smaller deviations: 1.025 versus 1.245 pp for LPBs and 1.160 versus 1.225 pp for UPBs. We remark that \HARP cannot obtain performance substantially better than \DAPRO for LPB construction tasks, as the results in~\citep{feldman2026spend} show that \DAPRO performs almost as well as an oracle alternative that has infinite budget on these tasks.

\begin{figure}[t]\centering
\paperfigure{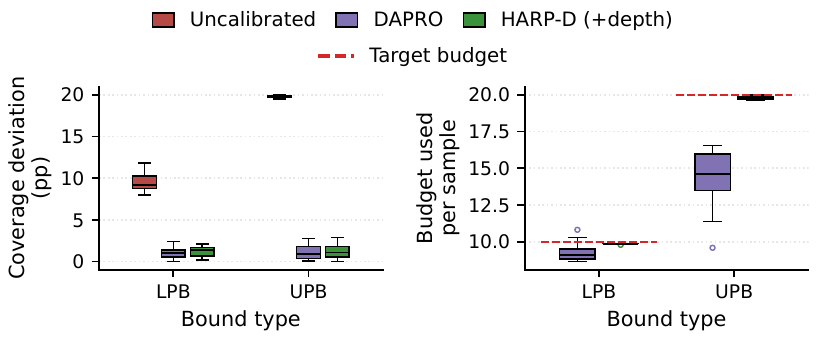}{2.4}
\caption{\textbf{AutoIF LPBs and UPBs, Qwen 2.5 14B.} Coverage deviation (left) and budget used per sample (right), over twenty runs. Target coverage: 90\% for LPB and 80\% for UPB; target budgets: 10 and 20 turns per sample, respectively (red dashed marks). \HARP\ uses D with depth for both bounds, selected by the criterion in Appendix~\ref{app:selection}. Uncalibrated has no acquisition-cost box.}
\label{fig:autoif}\end{figure}

\paragraph{Population metrics.}
We estimate the marginal event rate $\Prob(T \leq M)$ for the red teaming data with Llama Guard as a judge and report the performance of the methods in Figure~\ref{fig:red_team_event_rate}. While both calibration methods recover the correct metric in expectation, \HARP achieves 60.65\% lower variance across the four target models.

\begin{figure}[t]\centering
\paperfigure{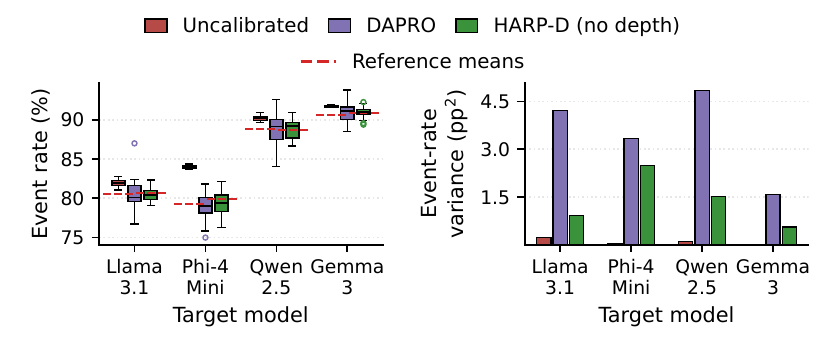}{2.3}
\caption{\textbf{Red-team event rates, Llama Guard judge.} Estimates and across-seed variances, all four targets, $\bar B=20$ turns per sample, twenty runs. \metricchoice\ Red marks denote benchmark-cohort means. Acquired methods use the same estimator.}
\label{fig:red_team_event_rate}\end{figure}

\paragraph{Budget.}

The mean budget utilized, i.e., consumed budget divided by available budget, by \HARP for LPB, UPB, and metrics is 94.51\%, 95.70\%, and 97.03\% compared to 91.58\%, 69.14\%, and 82.80\% of \DAPRO.
\HARP never exceeds the target budget across all experiments, while \DAPRO's overrun frequencies are 13.00\%, 3.00\%, and 12.75\%, respectively. This is not a surprise, as Theorem~\ref{thm:budget} guarantees a deterministic budget control of \HARP, while \DAPRO provides only an expected budget guarantee.
\FloatBarrier
\section{Discussion}
\label{sec:discussion}
We introduced \HARP for multi-turn allocation and calibration under hard budget constraints. 
Unlike \DAPRO~\citep{feldman2026spend}, \HARP does not require a policy-learning phase or a CRC budget-calibration step. Instead, it couples the censoring times of different samples. Despite this dependence, we show that the predictive bounds and estimated metrics of \HARP satisfy finite-sample guarantees. Our experiments support our theoretical claims and show that \HARP utilizes the budget effectively and yields estimates that have a low variance.

We remark that our guarantees require that the calibration and test data are i.i.d. with correct audit labels; in practice, labels may be noisy or partially missing, motivating extensions of \HARP that are robust to such corruptions~\citep{einbinder2024label, penso2024conformal, feldman2024robustconformalpredictionusing, feldman2026conformal}. Furthermore, \HARP constructs bounds for a single event time. In many applications, it is necessary to jointly evaluate the safety of modern AI systems such as agents, which calls for developing bounds for multiple event times~\citep{feldman2023calibrated}. Additionally, under smaller budgets, many event times may be censored. Model-generated or synthetic outcomes could complement censored events~\citep{angelopoulos2023prediction, angelopoulos2023ppi++, bashari2026general, bashari2026synthetic}. Finally, \HARP requires a minimal total budget $B\ge K(1+\eta)M$.
 
Based on our experiments, we recommend \HARP-D with depth for constructing predictive bounds and \HARP-D without depth for estimating metrics. Depth reduces the median LPB coverage variance but increases the variance of the metric estimates (see Appendix~\ref{app:ablation}).

\section*{Acknowledgments}
Y.~R. and S.~F. were supported by the European Union (ERC, SafetyBounds, 101163414). Views and opinions expressed are however those of the authors only and do not necessarily reflect those of the European Union or the European Research Council Executive Agency. Neither the European Union nor the granting authority can be held responsible for them.
Y.R. thanks the Career Advancement Fellowship, Technion.

\label{main-text-end}
\clearpage
\bibliographystyle{unsrtnat}
\bibliography{references}

\begin{thebibliography}{31}
\providecommand{\natexlab}[1]{#1}
\providecommand{\url}[1]{\texttt{#1}}
\expandafter\ifx\csname urlstyle\endcsname\relax
  \providecommand{\doi}[1]{doi: #1}\else
  \providecommand{\doi}{doi: \begingroup \urlstyle{rm}\Url}\fi

\bibitem[Gui et~al.(2024)Gui, Hore, Ren, and Barber]{gui2024conformalized}
Yu~Gui, Rohan Hore, Zhimei Ren, and Rina~Foygel Barber.
\newblock Conformalized survival analysis with adaptive cut-offs.
\newblock \emph{Biometrika}, 111\penalty0 (2):\penalty0 459--477, 2024.
\newblock \doi{10.1093/biomet/asad076}.

\bibitem[Davidov et~al.(2025)Davidov, Feldman, Shamai, Kimmel, and Romano]{davidovconformalized}
Hen Davidov, Shai Feldman, Gil Shamai, Ron Kimmel, and Yaniv Romano.
\newblock {Conformalized Survival Analysis for General Right-Censored Data}.
\newblock In \emph{International Conference on Learning Representations}, 2025.
\newblock URL \url{https://openreview.net/forum?id=JQtuCumAFD}.

\bibitem[Davidov et~al.(2026)Davidov, Feldman, Freidkin, and Romano]{davidov2026calibrated}
Hen Davidov, Shai Feldman, Gilad Freidkin, and Yaniv Romano.
\newblock Calibrated predictive lower bounds on time-to-unsafe-sampling in {LLM}s.
\newblock In \emph{The 29th International Conference on Artificial Intelligence and Statistics}, 2026.

\bibitem[Feldman and Romano(2026)]{feldman2026spend}
Shai Feldman and Yaniv Romano.
\newblock How many iterations to jailbreak? dynamic budget allocation for multi-turn llm evaluation.
\newblock \emph{arXiv preprint arXiv:2605.06605}, 2026.

\bibitem[Yamin et~al.(2025)Yamin, Kennedy, and Wilder]{pmlr-v286-yamin25a}
Khurram Yamin, Edward Kennedy, and Bryan Wilder.
\newblock Dependent randomized rounding for budget constrained experimental design.
\newblock In Silvia Chiappa and Sara Magliacane, editors, \emph{Proceedings of the Forty-first Conference on Uncertainty in Artificial Intelligence}, volume 286 of \emph{Proceedings of Machine Learning Research}, pages 4681--4700. PMLR, 21--25 Jul 2025.

\bibitem[Madow and Madow(1944)]{madow1944theory}
William~G Madow and Lillian~H Madow.
\newblock On the theory of systematic sampling, i.
\newblock \emph{The Annals of Mathematical Statistics}, 15\penalty0 (1):\penalty0 1--24, 1944.

\bibitem[Deville and Till{\'e}(2004)]{deville2004efficient}
Jean-Claude Deville and Yves Till{\'e}.
\newblock Efficient balanced sampling: the cube method.
\newblock \emph{Biometrika}, 91\penalty0 (4):\penalty0 893--912, 2004.

\bibitem[Gandhi et~al.(2006)Gandhi, Khuller, Parthasarathy, and Srinivasan]{gandhi2006dependent}
Rajiv Gandhi, Samir Khuller, Srinivasan Parthasarathy, and Aravind Srinivasan.
\newblock Dependent rounding and its applications to approximation algorithms.
\newblock \emph{Journal of the ACM (JACM)}, 53\penalty0 (3):\penalty0 324--360, 2006.

\bibitem[Koenker and Bassett(1978)]{QR}
Roger Koenker and Gilbert Bassett.
\newblock Regression quantiles.
\newblock \emph{Econometrica}, 46\penalty0 (1):\penalty0 33--50, 1978.
\newblock ISSN 00129682, 14680262.
\newblock URL \url{http://www.jstor.org/stable/1913643}.

\bibitem[Chung et~al.(2020)Chung, Neiswanger, Char, and Schneider]{beyond_pinball_loss}
Youngseog Chung, Willie Neiswanger, Ian Char, and Jeff Schneider.
\newblock Beyond pinball loss: Quantile methods for calibrated uncertainty quantification.
\newblock \emph{arXiv preprint arXiv:2011.09588}, 2020.

\bibitem[Feldman et~al.(2021)Feldman, Bates, and Romano]{feldman2021improving}
Shai Feldman, Stephen Bates, and Yaniv Romano.
\newblock Improving conditional coverage via orthogonal quantile regression.
\newblock \emph{Advances in neural information processing systems}, 34:\penalty0 2060--2071, 2021.

\bibitem[Vovk et~al.(2005)Vovk, Gammerman, and Shafer]{vovk2005algorithmic}
Vladimir Vovk, Alexander Gammerman, and Glenn Shafer.
\newblock \emph{{Algorithmic Learning in a Random World}}, volume~29.
\newblock Springer, 2005.
\newblock \doi{10.1007/b106715}.

\bibitem[Horvitz and Thompson(1952)]{Horvitz01121952}
Daniel~G. Horvitz and Donovan~J. Thompson.
\newblock A generalization of sampling without replacement from a finite universe.
\newblock \emph{Journal of the American Statistical Association}, 47\penalty0 (260):\penalty0 663--685, 1952.

\bibitem[Angelopoulos et~al.(2024)Angelopoulos, Bates, Fisch, Lei, and Schuster]{angelopoulos2024conformal}
Anastasios~Nikolas Angelopoulos, Stephen Bates, Adam Fisch, Lihua Lei, and Tal Schuster.
\newblock Conformal risk control.
\newblock In \emph{The Twelfth International Conference on Learning Representations}, 2024.
\newblock URL \url{https://openreview.net/forum?id=33XGfHLtZg}.

\bibitem[Byrka et~al.(2017)Byrka, Pensyl, Rybicki, Srinivasan, and Trinh]{byrka2016improved}
Jaros{\l}aw Byrka, Thomas Pensyl, Bartosz Rybicki, Aravind Srinivasan, and Khoa Trinh.
\newblock An improved approximation for k-median and positive correlation in budgeted optimization.
\newblock \emph{ACM Transactions on Algorithms (TALG)}, 13\penalty0 (2):\penalty0 1--31, 2017.

\bibitem[Einbinder et~al.(2024)Einbinder, Feldman, Bates, Angelopoulos, Gendler, and Romano]{einbinder2024label}
Bat-Sheva Einbinder, Shai Feldman, Stephen Bates, Anastasios~N Angelopoulos, Asaf Gendler, and Yaniv Romano.
\newblock Label noise robustness of conformal prediction.
\newblock \emph{Journal of Machine Learning Research}, 25\penalty0 (328):\penalty0 1--66, 2024.

\bibitem[Penso and Goldberger(2024)]{penso2024conformal}
Coby Penso and Jacob Goldberger.
\newblock A conformal prediction score that is robust to label noise.
\newblock \emph{arXiv preprint arXiv:2405.02648}, 2024.

\bibitem[Feldman and Romano(2024)]{feldman2024robustconformalpredictionusing}
Shai Feldman and Yaniv Romano.
\newblock Robust conformal prediction using privileged information.
\newblock \emph{Advances in Neural Information Processing Systems}, 37:\penalty0 117813--117852, 2024.

\bibitem[Feldman et~al.(2026)Feldman, Bates, and Romano]{feldman2026conformal}
Shai Feldman, Stephen Bates, and Yaniv Romano.
\newblock Conformal prediction with corrupted labels: Uncertain imputation and robust re-weighting.
\newblock In \emph{International Conference on Learning Representations}, volume 2026, pages 30901--30942, 2026.

\bibitem[Feldman et~al.(2023)Feldman, Bates, and Romano]{feldman2023calibrated}
Shai Feldman, Stephen Bates, and Yaniv Romano.
\newblock Calibrated multiple-output quantile regression with representation learning.
\newblock \emph{Journal of Machine Learning Research}, 24\penalty0 (24):\penalty0 1--48, 2023.

\bibitem[Angelopoulos et~al.(2023{\natexlab{a}})Angelopoulos, Bates, Fannjiang, Jordan, and Zrnic]{angelopoulos2023prediction}
Anastasios~N Angelopoulos, Stephen Bates, Clara Fannjiang, Michael~I Jordan, and Tijana Zrnic.
\newblock Prediction-powered inference.
\newblock \emph{Science}, 382\penalty0 (6671):\penalty0 669--674, 2023{\natexlab{a}}.

\bibitem[Angelopoulos et~al.(2023{\natexlab{b}})Angelopoulos, Duchi, and Zrnic]{angelopoulos2023ppi++}
Anastasios~N Angelopoulos, John~C Duchi, and Tijana Zrnic.
\newblock Ppi++: Efficient prediction-powered inference.
\newblock \emph{arXiv preprint arXiv:2311.01453}, 2023{\natexlab{b}}.

\bibitem[Bashari et~al.(2026{\natexlab{a}})Bashari, Lee, Lotan, Dobriban, and Romano]{bashari2026general}
Meshi Bashari, Yonghoon Lee, Roy~Maor Lotan, Edgar Dobriban, and Yaniv Romano.
\newblock General synthetic-powered inference.
\newblock In \emph{Forty-third International Conference on Machine Learning}, 2026{\natexlab{a}}.
\newblock URL \url{https://openreview.net/forum?id=sxLncu2Fhx}.

\bibitem[Bashari et~al.(2026{\natexlab{b}})Bashari, Lotan, Lee, Dobriban, and Romano]{bashari2026synthetic}
Meshi Bashari, Roy~Maor Lotan, Yonghoon Lee, Edgar Dobriban, and Yaniv Romano.
\newblock Synthetic-powered predictive inference.
\newblock \emph{Advances in Neural Information Processing Systems}, 38:\penalty0 117704--117764, 2026{\natexlab{b}}.

\bibitem[Gehman et~al.(2020)Gehman, Gururangan, Sap, Choi, and Smith]{gehman2020realtoxicityprompts}
Samuel Gehman, Suchin Gururangan, Maarten Sap, Yejin Choi, and Noah~A Smith.
\newblock {RealToxicityPrompts: Evaluating Neural Toxic Degeneration in Language Models}.
\newblock In \emph{Findings of the Association for Computational Linguistics: EMNLP}, pages 3356--3369, 2020.

\bibitem[Ganguli et~al.(2022)Ganguli, Lovitt, Kernion, Askell, Bai, Kadavath, Mann, Perez, Schiefer, Ndousse, et~al.]{ganguli2022red}
Deep Ganguli, Liane Lovitt, Jackson Kernion, Amanda Askell, Yuntao Bai, Saurav Kadavath, Ben Mann, Ethan Perez, Nicholas Schiefer, Kamal Ndousse, et~al.
\newblock Red teaming language models to reduce harms: Methods, scaling behaviors, and lessons learned.
\newblock \emph{arXiv preprint arXiv:2209.07858}, 2022.

\bibitem[Inan et~al.(2023)Inan, Upasani, Chi, Rungta, Iyer, Mao, Tontchev, Hu, Fuller, Testuggine, et~al.]{inan2023llama}
Hakan Inan, Kartikeya Upasani, Jianfeng Chi, Rashi Rungta, Krithika Iyer, Yuning Mao, Michael Tontchev, Qing Hu, Brian Fuller, Davide Testuggine, et~al.
\newblock Llama guard: Llm-based input-output safeguard for human-ai conversations.
\newblock \emph{arXiv preprint arXiv:2312.06674}, 2023.

\bibitem[Rajpurkar et~al.(2018)Rajpurkar, Jia, and Liang]{rajpurkar-etal-2018-know}
Pranav Rajpurkar, Robin Jia, and Percy Liang.
\newblock Know what you don{'}t know: Unanswerable questions for {SQ}u{AD}.
\newblock In Iryna Gurevych and Yusuke Miyao, editors, \emph{Proceedings of the 56th Annual Meeting of the Association for Computational Linguistics (Volume 2: Short Papers)}, pages 784--789, Melbourne, Australia, July 2018. Association for Computational Linguistics.

\bibitem[Dong et~al.(2025)Dong, Lu, Li, Xia, Yu, Zhou, and Zhou]{dong2025selfplay}
Guanting Dong, Keming Lu, Chengpeng Li, Tingyu Xia, Bowen Yu, Chang Zhou, and Jingren Zhou.
\newblock Self-play with execution feedback: Improving instruction-following capabilities of large language models.
\newblock In \emph{The Thirteenth International Conference on Learning Representations}, 2025.

\bibitem[Chao et~al.(2025)Chao, Robey, Dobriban, Hassani, Pappas, and Wong]{chao2025jailbreaking}
Patrick Chao, Alexander Robey, Edgar Dobriban, Hamed Hassani, George~J Pappas, and Eric Wong.
\newblock Jailbreaking black box large language models in twenty queries.
\newblock In \emph{2025 IEEE Conference on Secure and Trustworthy Machine Learning (SaTML)}, pages 23--42. IEEE, 2025.

\bibitem[Wang et~al.(2024)Wang, Yang, Huang, Yang, Majumder, and Wei]{wang2024multilingual}
Liang Wang, Nan Yang, Xiaolong Huang, Linjun Yang, Rangan Majumder, and Furu Wei.
\newblock Multilingual e5 text embeddings: A technical report.
\newblock \emph{arXiv preprint arXiv:2402.05672}, 2024.

\end{thebibliography}
\clearpage
\appendix
\section{Additional Related Work}
\label{app:related}
Inverse-probability estimation for unequal-probability sampling is classical~\citep{Horvitz01121952}. Calibrated event-time prediction under censoring also predates this work: for example, \citet{gui2024conformalized} develop adaptive cut-offs for conformalized survival analysis. These results motivate controlling the resolution event and candidate family. Model-assisted correction has a substantial history as well, with prediction-powered inference providing a recent general framework for combining predictions and observed outcomes \citep{angelopoulos2023prediction}. We use the sequential augmented estimator from~\citep{feldman2026spend}.

\citet{pmlr-v286-yamin25a} study experimental assignment with a fixed treatment count. Their dependent-rounding construction preserves first-order probabilities and supports unbiasedness, variance comparisons under specified conditions, and asymptotic inference. The distinction from our setting is therefore the constraint and subsequent inference problem. \HARP\ reserves heterogeneous maximum trajectory lengths, incurs random event-stopped costs, updates later policies from released resources, and calibrates a selected nested family of predictive bounds. We account for these features through the reservation invariant and predictable-group analysis below.

The weighted pairwise rounding properties we use are established in the budgeted-rounding literature, including \citet{byrka2016improved}. We provide a self-contained proof to specify the properties needed for calibration. In particular, the proof uses cylinder inequalities for products of selected or unselected coordinates. Our analysis applies these inequalities within a sequential ledger, together with sample-local continuation and subsequent predictive calibration.

\DAPRO~\citep{feldman2026spend} is our direct baseline. Its history-dependent score, inverse-probability correction, and augmented UPB construction remain relevant to our setting. \HARP\ additionally connects worst-case reservations to realized costs and justifies inference under the resulting dependence. 

\section{Complete Budget and LPB Proofs}
\subsection{Weighted pivotal rounding}
\label{app:pivot}
Fix a group, positive costs $d_i\le M$ and starting probabilities $y_i=p_i^c$. Coordinates initially at zero or one need no update. At each step, choose two fractional coordinates $i,j$ and set
\begin{align}
 a_+&=\min\{d_i(1-y_i),d_j y_j\},&
 a_-&=\min\{d_i y_i,d_j(1-y_j)\}.\label{eq:pivotsteps}
\end{align}
Draw $s=a_+$ with probability $a_-/(a_++a_-)$ and $s=-a_-$ otherwise, and replace $(y_i,y_j)$ by $(y_i+s/d_i,y_j-s/d_j)$. At least one coordinate becomes integral at each step. If one fractional coordinate remains, round it by a Bernoulli draw with its current value as success probability. Denote the resulting vector by $\zeta$.

\begin{lemma}[Properties of the weighted pivotal rounding]\label{lem:pivot}
Conditional on the input group, weighted pivotal rounding satisfies
\begin{align}
 \E\zeta_i&=p_i^c,&
 \sum_i d_i\zeta_i&\le\sum_i d_i p_i^c+\max_i d_i,\label{eq:pivotproperties}\\
 \E\prod_{i\in J}\zeta_i&\le\prod_{i\in J}p_i^c,&
 \E\prod_{i\in J}(1-\zeta_i)&\le\prod_{i\in J}(1-p_i^c)
 \quad\text{for every }J.\label{eq:cylinders}
\end{align}
In particular, for fixed $h_i\in[0,1]$,
\begin{equation}
 \E\prod_i(1-h_i\zeta_i)\le\prod_i(1-h_i p_i^c).\label{eq:product}
\end{equation}
\end{lemma}
\begin{proof}
Each pairwise step has $\E[s\mid y]=0$, so every coordinate is a martingale and the weighted sum is unchanged. The final Bernoulli step changes the weighted sum by at most $\max_i d_i\le M$. This proves~\eqref{eq:pivotproperties}.

For a fixed subset $J$, consider $\prod_{i\in J}y_i$. If neither updated coordinate is in $J$, the product does not change. If exactly one is in $J$, its conditional expectation is unchanged. If both are in $J$, their product changes by a linear term in $s$ minus $s^2/(d_i d_j)$. Multiplication by the other nonnegative coordinates preserves the nonpositive conditional drift. The same calculation for $(1-y_i)(1-y_j)$ gives a nonpositive quadratic term. Both subset products are therefore supermartingales until the last fractional coordinate. That final step is linear in its single coordinate and preserves their conditional means. This proves~\eqref{eq:cylinders} without requiring any independence of the output coordinates.

Finally, write $1-h_i\zeta_i=(1-h_i)+h_i(1-\zeta_i)$ and expand the product. Every coefficient in this expansion is nonnegative. Applying the second inequality in~\eqref{eq:cylinders} term by term proves~\eqref{eq:product}. Zero-cost samples can be set to one outside the pivot operation; they do not affect the cost bound or these inequalities.
\end{proof}

\subsection{The pathwise ledger invariant}
\label{app:budget}
\begin{proof}[Proof of Theorem~\ref{thm:budget}]
At the beginning of the first group, feasibility gives $L_0\ge K(1+\eta)M$. Suppose that before group $k$,
\begin{equation}
 L_{k-1}\ge(K-k+1)(1+\eta)M.\label{eq:invariant}
\end{equation}
Then $\kappa_k\ge(1+\eta)M$, and in particular $\kappa_k-M\ge\eta M>0$. Summing the cap in~\eqref{eq:commit} over the positive-cost samples gives
\[
 \sum_{i\in I_k}d_i p_i^c\le
 \frac{\#\{i\in I_k:d_i>0\}}{m_k}(\kappa_k-M)\le\kappa_k-M.
\]
Lemma~\ref{lem:pivot} implies that the selected horizons cost at most $\kappa_k$. Each selected sample executes no more than $d_i$ turns. Event stopping and D's continuation stopping can only lower that cost. Thus $S_k\le\kappa_k$ and
\[
 L_k=L_{k-1}-S_k\ge L_{k-1}-\kappa_k=(K-k)(1+\eta)M.
\]
This is the invariant for the next group. At $k=K$ it gives $L_K\ge0$, proving $\sum_k S_k=B-L_K\le B$. All inequalities are pointwise. Randomized depth is covered by conditioning on its chosen horizons, all of which remain in $[0,M]$.
\end{proof}

\begin{lemma}[Positive reach]\label{lem:positive}
Under the same conditions, $p_i^c\ge p_0=\min\{p_{\min},\eta/m\}$ for every positive-cost sample.
\end{lemma}
\begin{proof}
The cap before clipping at one is at least $\eta M/(m_kd_i)\ge\eta/m$. The proposal is at least $p_{\min}$, so their minimum with one is at least $p_0$. C introduces no further thinning. D has at most $M$ continuation factors, each at least $p_{\min}\le1$. Free samples have probability one.
\end{proof}

The specified continuation rule admits a stronger bound when $\omega<1$. Suppose $\lambda_k\le\Lambda<\infty$ for every group and the predicted remaining cost lies in $[0,M]$. The model probabilities give $\nu_i(t)\ge1-\omega$, hence $u_i(t)\ge u_0:=\min\{1,\sqrt{(1-\omega)/(\Lambda M)}\}$. Equation~\eqref{eq:continue} implies $a_i(t)\ge\min\{a_i(t-1),u_i(t)\}$, whether $u_i(t)$ is above or below the previous reach. Induction therefore yields
\begin{equation}
 \pi_i\ge\widetilde p_0:=\min\{p_0,u_0\}\quad\text{for the specified D policy}.\label{eq:policyfloor}
\end{equation}
The interpolation in Appendix~\ref{app:reference} remains within the reference multiplier grid, whose largest value therefore supplies $\Lambda$. The reported configuration uses $\Lambda=10^9$ and $\omega=1/2$. For $M=200$, this gives $\widetilde p_0\ge1.58\times10^{-6}$. Although considerably stronger than $p_0p_{\min}^M$, it remains too small for informative coarse uniform PAC corrections. For arbitrary continuations with only $P_i(t)\ge p_{\min}$, the product-of-floors bound still applies.

\subsection{A filtration that permits reflow}
\label{app:filtration}
Condition on the training data, fitted model, independent reference design, and independent group permutation. After this conditioning, write the iid full working trajectories as $W_1,\ldots,W_N$, with $X_i$ and $T_i$ measurable functions of $W_i$. Let $\calF_{k-1}$ contain those fixed design objects, all complete trajectories in previous groups, and all random coins used in those groups. The observed past is a sub-$\sigma$-field, so the actual ledger and $\lambda_k$ are $\calF_{k-1}$-measurable. Complete trajectories in group $k$ remain independent of $\calF_{k-1}$ and of each other.

Within the group, $p_i^c$ is a function of $(W_i,\calF_{k-1})$, in fact only of the initially available portion of $W_i$. Fix independent counterfactual continuation streams for each sample, whether or not that sample is ultimately committed. The stream and its own trajectory determine $U_i$, the indicator that continuation would resolve the sample if committed, and the product $q_i=\prod_{t\le b_i}P_i(t)$. For C set $U_i=q_i=1$. With $Z_i=\zeta_i U_i$ and $\pi_i=p_i^c q_i$,
\begin{equation}
 \E[Z_i\mid W_{I_k},\calF_{k-1}]=\pi_i,
 \quad
 \E\left[\frac{Z_iA_i(\tau)}{\pi_i}\,\middle|\,W_{I_k},\calF_{k-1}\right]=A_i(\tau).
 \label{eq:unbiased}
\end{equation}
The expectation here averages over pivot and continuation randomization; it does not imply independence of the $Z_i$. Each $q_i$ is a counterfactual product along that sample's full history, rather than the probability of reaching only a prematurely censored prefix. In practice, we need the full product only for samples with $Z_i=1$.

\subsection{Concentration and selection for LPBs}
\label{app:lpbproof}
For a fixed $\tau$, abbreviate $X_i^\tau=Z_iA_i(\tau)/\pi_i$ and $\mu=\mu(\tau)$. Conditional on $\calF_{k-1}$,
\begin{equation}
 \E X_i^\tau=\mu,\qquad
 \sigma_{ik}^2:=\operatorname{Var}(X_i^\tau\mid\calF_{k-1})
       =\E[A_i(\tau)/\pi_i\mid\calF_{k-1}]-\mu^2.
 \label{eq:moments}
\end{equation}
For any $s\ge0$, condition first on the current full trajectories and their counterfactual continuation streams. Apply~\eqref{eq:product} with $h_i=1-\exp\{-sU_iA_i(\tau)/\pi_i\}\in[0,1]$. Then integrate the independent continuation streams and current trajectories. Sample locality permits this final factorization, giving
\begin{equation}
 \E\left[e^{-s\sum_{i\in I_k}X_i^\tau}\mid\calF_{k-1}\right]
 \le\prod_{i\in I_k}\E[e^{-sX_i^\tau}\mid\calF_{k-1}].
 \label{eq:groupmgf}
\end{equation}

For completeness, if a centered random variable $Y\le1$ has variance $\sigma^2$, Taylor's formula with integral remainder implies
\[
 e^{sy}\le1+sy+(e^s-1-s)y^2\quad(y\le1,s\ge0).
\]
Indeed, the remainder divided by $s^2y^2$ is $\int_0^1(1-t)e^{tsy}\,dt$, at most its value at $y=1$. For $0<s<3$, the series bound $e^s-1-s\le s^2/[2(1-s/3)]$ therefore yields
\begin{equation}
 \E e^{sY}\le\exp\!\left\{\frac{s^2\sigma^2}{2(1-s/3)}\right\}.
 \label{eq:onesided}
\end{equation}
Apply this with $Y=\mu-X_i^\tau\le\mu\le1$. By~\eqref{eq:groupmgf}, the process
\begin{equation}
 \exp\left\{s\sum_{j\le k}\sum_{i\in I_j}(\mu-X_i^\tau)
       -\frac{s^2}{2(1-s/3)}\sum_{j\le k}\sum_{i\in I_j}\sigma_{ij}^2\right\}
 \label{eq:supermart}
\end{equation}
is a nonnegative supermartingale in the enlarged group filtration. Adaptive changes of the ledger are allowed because the group bound is conditional on its current past.

If $v(\tau)\le v_\star$ almost surely, the compensator sum is at most $N(v_\star-\mu^2)$. The exponential Markov inequality and optimization over $s$ give, for $x>0$,
\begin{equation}
 \Prob\{\mu-\muhat(\tau)\ge x\}
 \le\exp\left\{-\frac{Nx^2}{2(v_\star-\mu^2+x/3)}\right\}.
 \label{eq:lpbtail}
\end{equation}
The usual zero-variance case follows by a limit or directly from the deterministic centered sum. This is a lower-tail bound for a nonnegative HT estimator; no bounded upper weight is needed beyond existence of the stated second moments.

\begin{proof}[Proof of Theorem~\ref{thm:lpb}]
If no retained candidate has population risk greater than $\alpha+\Delta$, the assertion is automatic. Otherwise let $\tau^\dagger$ be the first such candidate. It is fixed after conditioning on the reference and model, even though it is unknown to the procedure. Since the population losses are nested, $\mu(\widehat\tau)>\alpha+\Delta$ implies $\widehat\tau\ge\tau^\dagger$. The inclusive-prefix rule then forces $\muhat(\tau^\dagger)\le\alpha$, and consequently
\[
 \mu(\tau^\dagger)-\muhat(\tau^\dagger)>\Delta.
\]
Use the exponential argument above at this single candidate, bounding its variance compensator by $N(v_\star-\alpha^2)$ since $\mu(\tau^\dagger)>\alpha$. The resulting upper bound is $\exp\{-N\Delta^2/[2((v_\star-\alpha^2)_++\Delta/3)]\}$, at most $\delta$ under~\eqref{eq:lpbdelta}. If $\alpha+\Delta\ge1$, the assertion is already trivial. Integrating over the conditioned reference/model preserves the probability bound, including when the declared envelope and $\Delta$ are functions of that reference.
\end{proof}

\subsection{Restricted moment envelopes and jumps}
\label{app:restricted}
Order the retained finite grid as $\tau_0=0<\cdots<\tau_J$. For $u\ge0$, define
\[
 \calT_+(u)=\{\tau_0\}\cup\{\tau_j:j\ge1,\ \mu(\tau_{j-1})\le u\}.
\]
Suppose $v_\sharp(u)$ bounds~\eqref{eq:varianceenvelope} almost surely at every candidate in $\calT_+(u)$. Theorem~\ref{thm:lpb} remains valid if~\eqref{eq:lpbdelta} uses $v_\sharp(\alpha+\Delta)$ in place of $v_\star$. Its proof is identical because the predecessor of $\tau^\dagger$ is safe and hence $\tau^\dagger\in\calT_+(\alpha+\Delta)$.

An envelope restricted to $\{\tau:\mu(\tau)\le u\}$ would miss the crossing candidate at a risk jump and would not justify this proof. Such jumps arise naturally with discrete event times and quantile grids. 

\subsection{Expected unused budget}
\label{app:shortfall}
We next quantify the effect of conditional pacing on unused budget.
\begin{proposition}[Pacing-defect identity]\label{prop:shortfall}
Let $S_k$ denote actual group spend and define the predictable defect
\[
 e_k=\frac{L_{k-1}}{K-k+1}-\E[S_k\mid\calF_{k-1}].
\]
For $K\ge2$,
\begin{equation}
 \E L_K\le\E L_{K-1}
 =\frac BK+\sum_{k=1}^{K-1}\frac{\E e_k}{K-k}
 \le\frac BK+\sum_{k=1}^{K-1}\frac{\E(e_k)_+}{K-k}.
 \label{eq:shortfall}
\end{equation}
In particular, exact conditional pacing in every nonfinal group implies $\E L_K\le B/K$. For $K=1$ the bound $\E L_1\le B$ is immediate.
\end{proposition}
\begin{proof}
Put $h=K-k+1$. From $L_k=L_{k-1}-S_k$,
\[
 \E[L_k\mid\calF_{k-1}]=\frac{h-1}{h}L_{k-1}+e_k.
\]
Taking expectations and iterating for $k=1,\ldots,K-1$ telescopes the coefficient of $B$ to $1/K$. The coefficient of $\E e_k$ is $\prod_{j=k+1}^{K-1}(K-j)/(K-j+1)=1/(K-k)$. This proves the equality. Nonnegative spend gives $L_K\le L_{K-1}$, and $e_k\le(e_k)_+$ gives the last inequality.
\end{proof}
Recycling savings does not ensure that the reference-based $\lambda_k$ eliminates the pacing defect. Commitment caps, inaccurate cost forecasts, continuation thinning, or a shortage of useful remaining turns can all affect realized spend. For example, suppose all $T_i=1$ and $B>N$. Any policy that stops at an event can spend at most $N$ turns on these $N$ samples, regardless of its ledger accuracy. This rules out an unconditional utilization guarantee. 

\section{Reference Planning and Complete Algorithms}
\label{app:algorithms}
\subsection{The predictions from the reference set}
\label{app:reference}
For LPBs, choose the initial design index closest to the target risk:
\[
 \tau_\star^{\rm full}\in\arg\min_{\tau\le\prior}
       \left|R^{-1}\sum_{r\in\refdata}\widehat\Prob(T_r<\fhat_\tau(X_r)\mid X_r)-\alpha_{\rm design}\right|.
\]
We resolve ties using the first grid index. Once LPB depth $\ell$ is chosen, we replace the design index by $\min\{\tau_\star^{\rm full},\ell\}$ in grid order. For constructing LPBs, we set $\alpha_{\rm design}=.1$.

We construct the reference set using only initial PMFs. For C at multiplier $\lambda$, the expected surrogate cost is the mean of $p_r^{\rm prop}(\lambda)\widehat c_r$. For D, we condition each initial PMF on survival to the current prefix, without accessing future reference histories. We apply the continuation rule along this synthetic survival-conditioned path and sum the predicted reached-turn costs:
\begin{equation}
 c_{\refdata}(\lambda)=\frac1R\sum_{r\in\refdata}\sum_{t=1}^{d_r}
 \widehat\Prob(T_r\ge t\mid X_r)\,
 p_r^{\rm prop}(\lambda)\prod_{s=1}^t P_r^{\rm ref}(s;\lambda).
 \label{eq:refcost}
\end{equation}
This expression includes commitment and the first continuation draw. It forecasts the uncapped proposal. We use 192 logarithmically spaced multipliers from $10^{-9}$ to $10^9$, make the cost curve nonincreasing, and interpolate $\log\lambda$ against reversed cost to match $\beta_k/m_k$. Targets outside the tabulated cost range use the corresponding endpoint multiplier. Both reference model error and the subsequent commitment cap can therefore prevent exact conditional pacing.

For a working LPB sample that has survived through $t-1$, the design risk in~\eqref{eq:continue} is
\[
 \nu_i(t)=(1-\omega)+\omega\,
 \frac{\widehat\Prob\{T_i<\fhat_{\tau_\star}(X_i)\mid H_i(t-1),T_i\ge t\}}
      {\max(\bar r,10^{-12})}.
\]
We predict remaining cost as $\widehat e_i(t)=\widehat\E[\min(T_i-t+1,d_i-t+1)\mid H_i(t-1),T_i\ge t]$. Nonpositive remaining lengths have zero cost, and the policy is called only for required positive turns. We safeguard against denominators and square-root inputs by adding $10^{-12}$. 

\subsection{The LPB depth planner}
\label{app:depth}
Let the grid indices run from $0$ to $J$, with $0$ the zero bound. From the initial reference distribution, compute nondecreasing curves
\begin{align}
 c_j&=R^{-1}\sum_r\widehat\E[\min(T_r,(\fhat_{\tau_j}(X_r)-1)_+)\mid X_r],\\
 \widehat\mu_j&=R^{-1}\sum_r\widehat\Prob(T_r<\fhat_{\tau_j}(X_r)\mid X_r),\qquad
 s_j=R^{-1}\sum_r\fhat_{\tau_j}(X_r).
\end{align}
Use $\Gamma=\{.6,.8,1,1.25,1.6\}$, $\delta_{\rm plan}=.1$, $\epsilon_{\rm plan}=.02$, $B'= (1-\epsilon_{\rm plan})B/N$, and $g_j=\min\{1,B'/\max(c_j,10^{-12})\}$. Define the surrogate width and attainable value by
\begin{align}
 U_{j,r}(\gamma)&=\gamma\widehat\mu_r+
       \sqrt{\frac{2\log(5/\delta_{\rm plan})\gamma\widehat\mu_r}{g_jN}}
       +\frac{\log(5/\delta_{\rm plan})}{3g_jN},\\
 V_j(\gamma)&=\max\bigl(\{s_r:r\le j,\ U_{j,r}(\gamma)\le\alpha_{\rm design}\}\cup\{0\}\bigr).
\end{align}
Among $j=1,\ldots,J$, choose the first minimizer
\[
 \widetilde j=\arg\min_j\max_{\gamma\in\Gamma}
               \{\max_hV_h(\gamma)-V_j(\gamma)\}.
\]
Let $j_{\rm aff}=\max\{j:c_j\le B'\}$, with zero allowed, and set $\widehat j=\min\{J,\max(1,\widetilde j,j_{\rm aff})\}$. The retained endpoint is $\ell=\tau_{\widehat j}$. The quantities $g_j$ and $U_{j,r}$ are only planning surrogates: the actual resolution probability is determined by the ledger, pivot, and, for D, continuation. 

\begin{algorithm}[t]
\caption{Weighted pivotal commitments (positive-cost samples)}
\label{alg:pivot}
\begin{algorithmic}[1]
\REQUIRE Probabilities $p_i^c\in[0,1]$, positive costs $d_i\le M$
\STATE Set $y\gets p^c$.
\WHILE{at least two coordinates are fractional}
\STATE Pick two fractional coordinates $i,j$; compute $a_+,a_-$ by~\eqref{eq:pivotsteps}.
\STATE Draw $s=a_+$ with probability $a_-/(a_++a_-)$, otherwise $s=-a_-$.
\STATE Set $y_i\gets y_i+s/d_i$, $y_j\gets y_j-s/d_j$.
\ENDWHILE
\IF{one fractional coordinate $i$ remains}
\STATE Set $y_i\gets\operatorname{Bernoulli}(y_i)$.
\ENDIF
\RETURN $\zeta=y$.
\end{algorithmic}
\end{algorithm}

\begin{algorithm}[t]
\caption{\HARP-C/D execution with exact acquisition records}
\label{alg:full}
\begin{algorithmic}[1]
\REQUIRE Working initial features, fixed reference set, horizons $d_i$, cap $B$
\STATE Independently permute samples; form groups of size at most $m$; verify~\eqref{eq:feasible}.
\STATE Set $L=B$ and initialize all observed-turn counts to zero.
\FOR{$k=1,\ldots,K$}
\STATE Compute $\kappa_k,\beta_k$ from~\eqref{eq:ledger}; choose reference multiplier $\lambda_k$.
\STATE Compute $p_i^c$ by~\eqref{eq:commit}; set $p_i^c=\zeta_i=1$ for $d_i=0$.
\STATE Round positive-cost samples with Algorithm~\ref{alg:pivot}; set $S_k=0$.
\FOR{each sample $i\in I_k$}
\STATE Set $a_i=p_i^c$; record commitment and horizon; if $d_i=0$, mark resolved.
\IF{$\zeta_i=1$ and $d_i>0$}
\FOR{$t=1,\ldots,d_i$ until an event or a continuation stop}
\STATE For C set $P_i(t)=1$; for D use~\eqref{eq:continue} from the acquired prefix.
\STATE Record $P_i(t)$ and update $a_i\gets a_iP_i(t)$.
\STATE Draw continuation; if it fails, stop this sample without paying for turn $t$.
\STATE Otherwise acquire turn $t$, increment the sample's count and $S_k$, and update history.
\STATE If the event occurs or $t=d_i$, mark resolved and store $\pi_i=a_i$.
\ENDFOR
\ENDIF
\ENDFOR
\STATE Set $L\gets L-S_k$; retain the records needed for the chosen estimator.
\ENDFOR
\RETURN Acquired prefixes, commitment/continuation probabilities, resolution flags, total spend.
\end{algorithmic}
\end{algorithm}
We implement continuation with independent random uniforms; C can omit draws with probability one. An unresolved sample contributes zero to the LPB estimator. For sequential AHT, each reached turn uses its own recorded cumulative probability, including commitment at the first turn. Although a replay cache may store full histories, the policy must not access future predictions.

\section{Upper Predictive Bounds}
\label{app:upb}
\subsection{Candidate family, acquisition, and selection}
For each finite candidate $f_i=\fhat_\tau(X_i)\le M$, define $A_i^{\rm U}(\tau)=\ind{T_i>f_i}$. When $f_i=\infty$, we set $A_i^{\rm U}=0$ deterministically. We represent infinity by 201 when $M=200$; this value is not an observed event on turn 201. We retain the grid through $\prior$, together with the last, all-infinite candidate, excluding finite candidates beyond the acquisition prior. Every retained finite candidate therefore satisfies $f_i\le d_i=\min\{\fhat_{\prior}(X_i),M\}$.

We choose the reference anchor $\tau_a\le\prior$ as the first candidate for which the entire subsequent reference-predicted coverage suffix reaches the target $1-\alpha$. If no candidate satisfies this condition, we use $\prior$. Unlike the LPB closest-risk anchor, this rule uses a coverage suffix. The design score also changes: we predict \emph{survival beyond} $\fhat_{\tau_a}(X_i)$ from the reference and observed history, rather than an event strictly before an LPB endpoint. An infinite anchor has zero initial risk. We then apply the same smoothing and cost-based continuation formula. 

Without depth, each sample has acquisition horizon $d_i$. For \HARP with depth version, let $a_i=\min\{\fhat_{\tau_a}(X_i),d_i\}$, clipped to $[1,M]$, and draw independently
\begin{equation}
 D_i=\begin{cases}d_i,&\text{with probability }q,\\a_i,&\text{with probability }1-q,\end{cases}
 \qquad q=.25.\label{eq:insured}
\end{equation}
We use $D_i$ as both commitment cost and horizon in Algorithm~\ref{alg:full}, and form the reference cost curve from the corresponding $q/(1-q)$ mixture of full and shallow forecasts. This reduces late acquisition while retaining every finite target and the reported UPB family. Define the depth-tail function
\begin{equation}
 G_i(t)=\Prob(D_i\ge t\mid X_i)=
 \begin{cases}1,&t\le a_i,\\q,&a_i<t\le d_i.\end{cases}\label{eq:depthtail}
\end{equation}
The commitment and continuation probabilities can depend on the sampled $D_i$. Conditional on that depth and the complete trajectory, write their product through turn $t$ as $r_i^{D_i}(t)$ when $t\le D_i$. Beyond $D_i$, extend the product positively, for example, with continuation factors of one; those terms are never observed. The effective weighting quantity is
\begin{equation}
 \rho_i^{D_i}(t)=G_i(t)r_i^{D_i}(t).\label{eq:effectivereach}
\end{equation}
Since $r_i^{D_i}$ can differ between the two horizons, this quantity is generally \emph{not} the unconditional marginal reach probability. Validity instead follows from the weighting identity below. To represent $G_i(t)$, we multiply the recorded conditional probability at turn $t$ by $G_i(t)/G_i(t-1)$, using $G_i(0)=1$.

For a finite candidate, we initialize $m_{i0}(\tau)=\widehat\Prob(T_i>f_i\mid X_i)$. If an event occurs by $f_i$, we set the post-event prediction to zero; if the sample survives through $f_i$, we set it to one. Before resolution, we use the pre-trained model at the reached history, and after resolution, we keep the value fixed. For an infinite candidate, all predictions and the loss remain zero. We apply~\eqref{eq:aht} until resolution or acquisition stopping and average sample contributions to obtain $\muhat_{\rm U}(\tau)$.

Although population miscoverage decreases with $\tau$, the signed empirical curve need not be decreasing. We therefore select the \emph{smallest} candidate satisfying
\begin{equation}
 \max_{u\ge\tau}\muhat_{\rm U}(u)\le\alpha,\label{eq:upbselector}
\end{equation}
where the maximum includes the retained finite candidates and the infinite fallback. Equivalently, we require the minimum empirical coverage over the suffix to be at least $1-\alpha$. The selector uses this ordering, rather than an isolated crossing of the raw empirical curve.

\subsection{Unbiasedness with depth-dependent acquisition}
\begin{lemma}[Sequential weighting identity]\label{lem:aht}
Fix the past groups and full current trajectories. Assume the independent depth draw has a positive tail $G_i(t)$ on each required prefix, and the model prediction path is a causal function of the trajectory and fixed design, not of the sampled depth. Then the estimator in~\eqref{eq:aht}, with~\eqref{eq:effectivereach}, satisfies
\[
 \E[V_i(\tau)\mid W_{I_k},\calF_{k-1}]=A_i^{\rm U}(\tau).
\]
\end{lemma}
\begin{proof}
Consider a turn $t\le\min(T_i,f_i)$ where a prediction increment can be nonzero. Conditional on all current depths and full trajectories, the marginal commitment property and the independent continuation stream give
\[
 \E[R_{it}\mid D,W_{I_k},\calF_{k-1}]
       =\ind{D_i\ge t}\,r_i^{D_i}(t).
\]
After division by $G_i(t)r_i^{D_i}(t)$ and averaging over the independent depth draw, the expectation is one. Depth-dependent commitment probabilities cancel \emph{inside} that expectation. The prediction increment is fixed by the trajectory and is independent of the depth draw, so every term has expectation $m_{it}-m_{i,t-1}$. Their sum telescopes from $m_{i0}$ to the terminal value $A_i^{\rm U}$. Later increments are zero. Infinite candidates are deterministically zero. This proof requires correct design probabilities and correct resolution.
\end{proof}
For the superpopulation theorem, we reveal the current group's independent depth coins only when processing that group in the proof filtration. Practically, these coin tosses can be pre-sampled without requiring the proof to condition on all future draws.

\subsection{A valid range bound for signed updates}
\begin{lemma}[Range of a sequential augmented contribution]\label{lem:range}
Let $m_0,\ldots,m_J\in[0,1]$ and $1\le w_1\le\cdots\le w_J\le W$. Then
\[
 1-W\le m_0+\sum_{t=1}^Jw_t(m_t-m_{t-1})\le W.
\]
The same interval contains $m_0$ when $J=0$.
\end{lemma}
\begin{proof}
For $J\ge1$, summation by parts gives
\[
 (1-w_1)m_0+\sum_{t=1}^{J-1}(w_t-w_{t+1})m_t+w_Jm_J.
\]
All coefficients except the last are nonpositive and sum to $1-w_J$. Their minimum contribution is $1-w_J$, while the positive coefficient contributes at most $w_J$. Hence the expression lies in $[1-w_J,w_J]\subseteq[1-W,W]$. This argument is sharper than separately bounding all $J$ increments by $W$.
\end{proof}
Every observed sample consists of a prefix. Along that prefix,~\eqref{eq:effectivereach} is nonincreasing because both $G_i(t)$ and the conditional-depth reach are nonincreasing, so the range lemma applies to each acquired sample. A sufficient uniform choice is
\[
 W=\begin{cases}
 (q p_0)^{-1},&\text{C with insured depth},\\
 (q p_0p_{\min}^{M})^{-1},&\text{D with insured depth},
 \end{cases}
\]
with $q=1$ when depth is not used. 
For the specified D policy,~\eqref{eq:policyfloor} improves the sufficient choice to $W=(q\widetilde p_0)^{-1}$.

\subsection{Signed pivotal concentration and adaptive reflow}
\label{app:signedmgf}
To analyze signed updates, we use the structure of the UPB estimator rather than treating each group as an arbitrary bounded sum. Fix a candidate and the group past. For each sample, introduce a latent variable $U_i$ containing its complete trajectory, depth draw, and independent counterfactual continuation stream. These variables are conditionally independent across samples. The commitment probabilities $p_i^c$ and horizons are functions of $U_i$ and the past, computed using only the information available to the procedure. We can then write
\begin{equation}
 V_i=m_{i0}(U_i)+\zeta_i h_i(U_i),\label{eq:affinesigned}
\end{equation}
where $h_i$ sums the weighted prediction increments that would be acquired if sample $i$ were committed. The coefficient $h_i$ may have either sign, but is independent of the pivot coins conditional on $U_i$ and the past. This representation applies to C and D because within-sample stopping never depends on another current sample's history and the group multiplier is fixed. It also applies to insured depth: the draw belongs to $U_i$ and does not change the model prediction path.

\begin{lemma}[A signed exponential comparison]\label{lem:signedmgf}
Under this sample-local representation and the two cylinder inequalities of Lemma~\ref{lem:pivot}, for every real $s$ for which the marginal moments exist,
\begin{equation}
 \E\left[e^{-s\sum_{i\in I_k}V_i}\mid\calF_{k-1}\right]
 \le\prod_{i\in I_k}
       \left(\E[e^{-2sV_i}\mid\calF_{k-1}]\right)^{1/2}.
 \label{eq:signedmgf}
\end{equation}
\end{lemma}
\begin{proof}
Condition on every $U_i$ and the group past, so all coefficients and probabilities are fixed. For $s\ge0$, split indices into $J_+=\{i:h_i\ge0\}$ and $J_-=\{i:h_i<0\}$. Cauchy--Schwarz gives
\[
 \E_\zeta e^{-s\sum_i\zeta_i h_i}
 \le\left(\E_\zeta e^{-2s\sum_{i\in J_+}\zeta_i h_i}\right)^{1/2}
     \left(\E_\zeta e^{-2s\sum_{i\in J_-}\zeta_i h_i}\right)^{1/2}.
\]
The first expectation is bounded by $\prod_{i\in J_+}(1-p_i^c+p_i^ce^{-2sh_i})$ using the down-cylinder expansion in~\eqref{eq:product}. In the second, each factor is $1+(e^{-2sh_i}-1)\zeta_i$, with a nonnegative coefficient. Expansion and the up-cylinder inequality give the same product form for $J_-$. Including the fixed baselines therefore bounds the full conditional expectation by
\[
 \prod_i\left[(1-p_i^c)e^{-2sm_{i0}}+p_i^ce^{-2s(m_{i0}+h_i)}\right]^{1/2}.
\]
Each bracket is a function of its own $U_i$ and the past. Conditional independence of the $U_i$ factors the expectation of their product. Jensen's inequality for the square root bounds each factor by the square root of its expectation. That expectation is the actual marginal $\E[e^{-2sV_i}\mid\calF_{k-1}]$, because the pivot preserves the conditional marginal $p_i^c$. This proves~\eqref{eq:signedmgf}. For $s<0$, interchange the two sign sets.
\end{proof}

\begin{proof}[Proof of Theorem~\ref{thm:upb}]
By Lemma~\ref{lem:aht}, $\E[V_i\mid\calF_{k-1}]=\mu_{\rm U}(\tau)=\mu$ for a fixed candidate. Its range has length $a=2W-1$ by Lemma~\ref{lem:range}. Conditional Hoeffding's lemma gives $\E e^{t(\mu-V_i)}\le e^{t^2a^2/8}$; this follows by bounding the exponential on the range interval by its endpoint chord and maximizing the centered two-point moment generating function. Apply~\eqref{eq:signedmgf} and bound each centered marginal mgf at $t=2s$ to obtain
\[
 \E\left[e^{s\sum_{i\in I_k}(\mu-V_i)}\mid\calF_{k-1}\right]
 \le e^{s^2m_ka^2/4}.
\]
Iterated conditional expectations over the ledger filtration give $\E e^{s(N\mu-\sum_iV_i)}\le e^{s^2Na^2/4}$. Exponential Markov with $s=2x/a^2$ yields
\begin{equation}
 \Prob\{\mu-N^{-1}\sum_iV_i\ge x\}\le e^{-Nx^2/a^2}.\label{eq:upbtail}
\end{equation}
Reverse the candidate order, starting at infinity. Population loss is now nondecreasing and~\eqref{eq:upbselector} is an inclusive-prefix rule. Unsafe selection requires underestimating the fixed first population-unsafe candidate by more than $\Delta_{\rm U}$. Apply~\eqref{eq:upbtail} there and integrate over the independent reference, as in the LPB proof.
\end{proof}

\begin{corollary}[Variance-sensitive signed-UPB bound]\label{cor:upbvariance}
Under the same assumptions, suppose a reference-measurable envelope $v_{\rm U}$ satisfies, almost surely at every retained candidate,
\[
 \frac1N\sum_k\sum_{i\in I_k}\E[V_i(\tau)^2\mid\calF_{k-1}]\le v_{\rm U}.
\]
Then~\eqref{eq:pac} also holds for any correction satisfying
\begin{equation}
 \Delta\ge\frac{2WL}{3N}+
    \sqrt{\frac{4W^2L^2}{9N^2}+\frac{4(v_{\rm U}-\alpha^2)_+L}{N}},
 \qquad L=\log(1/\delta).\label{eq:upbvariance}
\end{equation}
\end{corollary}
\begin{proof}
Since $\mu-V_i\le W$, the scaled one-sided Bernstein bound~\eqref{eq:onesided} applies. Combining it with~\eqref{eq:signedmgf} gives a conditional group mgf bounded by
\[
 \exp\left\{\frac{s^2\sum_{i\in I_k}\operatorname{Var}(V_i\mid\calF_{k-1})}
                    {1-2sW/3}\right\},\qquad 0<s<3/(2W).
\]
The sum of marginal conditional variances over groups is at most $N(v_{\rm U}-\mu^2)$. The exponential-supermartingale argument used for LPBs, now with variance factor two and scale $2W/3$, yields a lower-tail bound $\exp\{-Nx^2/[4(v_{\rm U}-\mu^2+Wx/3)]\}$. At the first unsafe candidate, replace $\mu^2$ by $\alpha^2$ and solve the resulting quadratic inequality to obtain~\eqref{eq:upbvariance}.
\end{proof}
We may take the minimum of the range-only and variance-sensitive corrections, since each bounds the same fixed-candidate failure event deterministically conditional on the reference. Treating each group as an arbitrary dependent block would instead give $(2W-1)\sqrt{\sum_km_k^2L/(2N^2)}$. That argument requires only conditional unbiasedness and the sample range. The sharper signed comparison also requires the sample-local affine representation and both cylinder inequalities, which hold for \HARP.

The signed comparison does not imply negative covariance between the $V_i$. For example, selecting exactly one of two equally likely items and multiplying the indicators by opposite signed coefficients makes contribution covariance positive, even though selection covariance is negative. The resulting variance is twice the sum of independent marginal variances. Thus, reusing the LPB product argument for signed updates without an additional cost would be incorrect. An independent-sample sequential-AHT variance diagnostic may provide marginal terms. Both valid UPB bounds can still be vacuous with small propensities.

\begin{algorithm}[t]
\caption{UPB construction under the \HARP acquisition design}
\label{alg:upb}
\begin{algorithmic}[1]
\REQUIRE UPB quantiles including infinity, target $1-\alpha$, reference initial PMFs
\STATE Retain finite candidates through the prior and the all-infinite fallback.
\STATE Select reference anchor by the predicted coverage suffix rule; set full horizons $d_i$.
\STATE If using depth, draw $D_i$ by~\eqref{eq:insured}; otherwise set $D_i=d_i$, $G_i=1$.
\STATE Build the reference cost set and run Algorithm~\ref{alg:full} with horizons $D_i$.
\FOR{each working sample and retained finite candidate}
\STATE Start from $m_{i0}$; for each reached required turn add $(m_{it}-m_{i,t-1})/\rho_i^{D_i}(t)$.
\STATE Impose loss zero after an event or one on survival to the finite bound; freeze thereafter.
\ENDFOR
\STATE Set infinite-candidate contributions to zero; average over all $N$ working samples.
\RETURN The smallest candidate satisfying~\eqref{eq:upbselector}.
\end{algorithmic}
\end{algorithm}

\section{Population-Level Metric Estimation}
\label{app:metrics}
Following \DAPRO's evaluation setup, we consider full-horizon event probability and restricted mean event time:
\begin{equation}
 p_M=\Prob(T\le M),\qquad r_M=\E[\min(T,M)].\label{eq:metrics}
\end{equation}
For harmful events, these quantities describe failure probability and time to failure. In AutoIF, the event is successful completion, so the interpretation is reversed. Neither target depends on the selected LPB endpoint. The mean event time conditional on $T\le M$ is a separate, third estimand.

For metric acquisition, we use design event $T\le M$ and full horizon $M$, retaining the same commitment and continuation rules for C and D. With depth, the LPB reference planner supplies a shallow horizon $a_i=\clip((\fhat_\ell(X_i)-1)_+,1,M)$, with probability $.25$ of acquiring to $M$. Applying the same $G_i(t)$ correction as in~\eqref{eq:depthtail} preserves the full-horizon target. The planner optimizes a surrogate for attainable LPB size, however, rather than metric variance.

Let $R_{iT}$ indicate that an event at $T_i\le M$ was observed and let $\rho_i(T_i)$ be its effective reach weight. Ordinary unbiased estimators are
\begin{equation}
 \widehat p_M^{\rm HT}=\frac1N\sum_i\frac{R_{iT}\ind{T_i\le M}}{\rho_i(T_i)},\qquad
 \widehat r_M=M-\frac1N\sum_i
       \frac{R_{iT}(M-T_i)\ind{T_i\le M}}{\rho_i(T_i)}.
 \label{eq:metricHT}
\end{equation}
We assign zero contribution to unobserved events, so no unknown event propensity is needed for them. The complement form for $\widehat r_M$ also avoids requiring a fully observed event-free trajectory, using the identity $\min(T,M)=M-(M-T)\ind{T\le M}$. We use this same terminal weighted construction for restricted mean under both acquired methods.

For event rate, both acquired methods use sequential augmentation. In~\eqref{eq:aht}, we let $m_{it}$ predict $\ind{T_i\le M}$, set it to one upon observing an event and to zero upon event-free survival through $M$, and average $V_i$. The depth-averaged telescoping proof in Lemma~\ref{lem:aht} applies unchanged. We apply this estimator to both acquired methods.

For a fixed finite population, the weighting identity gives design unbiasedness for the working-sample mean, conditional on the complete benchmark and reference split. Averaging over a uniform reference split recovers the parent-cohort mean. With iid working trajectories, it also gives population unbiasedness for~\eqref{eq:metrics}; the budget theorem applies in either setting. For a bounded sequential augmented metric, the proof of Theorem~\ref{thm:upb} gives a two-sided population bound with $\log(2/\delta)$ replacing $\log(1/\delta)$, without a candidate-selection argument. A normalized metric in $[0,1]$ uses the same range lemma. 

For conditional event time, form the weighted ratio
\[
 \widehat t_{\rm cond}=
 \frac{\sum_iR_{iT}T_i\ind{T_i\le M}/\rho_i(T_i)}
      {\sum_iR_{iT}\ind{T_i\le M}/\rho_i(T_i)}
\]
when its denominator is positive. Although both numerator and denominator are unbiased totals, their ratio is not generally unbiased. We report a zero denominator as undefined, rather than assigning a value of zero.

\begin{algorithm}[t]
\caption{Population metrics from one \HARP acquisition}
\label{alg:metrics}
\begin{algorithmic}[1]
\STATE Freeze the full target horizon $M$ and an independent initial-prediction reference design.
\STATE Use $d_i=M$; optionally draw insured shallow/full horizons with positive full-depth mass.
\STATE Run Algorithm~\ref{alg:full}; retain reached prefixes and their effective weights.
\STATE Estimate the event probability by the sequential event-target version of~\eqref{eq:aht}.
\STATE Estimate restricted mean by~\eqref{eq:metricHT}.
\end{algorithmic}
\end{algorithm}

\section{Experimental Details}
\label{app:experiments}
\subsection{Benchmarks, event definitions, and data generation}
We use the datasets, generated conversations, event labels, and prediction caches from~\citep{feldman2026spend} without modification. The five settings are RealToxicityPrompts toxicity~\citep{gehman2020realtoxicityprompts}, Anthropic red teaming~\citep{ganguli2022red} with a Qwen judge or Llama Guard~\citep{inan2023llama}, SQuAD-based retrieval-grounded hallucination~\citep{rajpurkar-etal-2018-know}, and AutoIF instruction following~\citep{dong2025selfplay}. In safety settings, the event is the first response judged harmful or ungrounded; in AutoIF, it is the first response passing all executable checks. The two red-team judges give distinct evaluation configurations. We use one-based event times, with 201 representing no event within the $M=200$ paid-turn horizon.

Qwen 2.5 14B is the attacker/helper in the inherited PAIR-style conversation setup~\citep{chao2025jailbreaking}; Section~\ref{sec:experiments} lists the four targets. We retain \DAPRO's generation hyperparameters, judge parsing, and retry rules without retuning them for \HARP. Each configuration has 10,000 initial prompts, split into 4,000 for model training and 6,000 for acquisition and calibration evaluation. Conversation generation and model fitting are unchanged.

\subsection{Prediction model}
We embed text histories with Multilingual-E5-large \citep{wang2024multilingual}. The model projects these embeddings to 258 dimensions with ReLU and dropout .2, then applies three pre-normalized Transformer layers with six attention heads, feed-forward width 1,032, GELU, and a causal attention mask. A multilayer perceptron maps the concatenated representations to survival logits. The conditional output is an $M\times(M+1)$ PMF grid. Its final outcome class represents no event through $M$, and its support is restricted to times after the observed prefix.

Training uses Adam with learning rate $10^{-4}$, weight decay .001, and batch size 64 for at most 500 epochs. Ten percent of the training set is reserved for validation, with early-stopping patience of 50. We fit a time-dependent temperature using L-BFGS on the validation survival negative log-likelihood. During \HARP runs, row $t$ of the conditional PMF cache becomes available only after acquiring the first $t$ interactions. Reference planning uses row zero alone, conditioning this initial PMF on survival to construct surrogate later predictions. These proxies might not be the same as the model's predictions after observing a real future response.

\subsection{Acquisition settings and comparison accounting}
\begin{table}[t]\centering
\caption{Target budgets in paid interaction turns per parent calibration sample. Each setting uses the same budget for all four target models; both acquired methods include their planning and control costs.}
\label{tab:settings}
\begin{tabular}{lrrr}\toprule
Dataset/judge & LPB & UPB & Metrics\\\midrule
Toxicity & 20 & 20 & 20\\
Red team, Qwen & 10 & 20 & 20\\
Red team, Llama Guard & 10 & 20 & 20\\
Hallucination & 10 & 20 & 20\\
AutoIF & 10 & 20 & 20\\\bottomrule
\end{tabular}\end{table}
We use the LPB budgets from the corresponding \DAPRO\ comparison. For full-horizon metrics and UPBs, we use twenty turns per prompt throughout, retaining access to late outcomes while staying well below full-observation cost. Budget choice does not depend on the selected \HARP\ variant. With $N_0=2,000$, $R=100$, $m=25$, $M=200$ and $\eta=.05$, the reserve is $76\cdot210=15,960$ turns, or 7.98 per parent prompt, so both prescribed budgets are feasible. We use a continuation/proposal floor of .02, risk smoothing $\omega=.5$, and full-depth probability .25 for insured depth.

For \DAPRO, we use the two-bin, current-hazard soft-prefix policy with global regularization .001 and sample-cost cap $2\bar B$ in the independent CRC controller. Its first stage has fifty paid samples: twenty-five for fitting the policy and twenty-five for CRC. \HARP's independent reference set uses initial predictions, incurs no interaction cost, and is excluded from estimation. We normalize every cost by the same parent size; \HARP's estimator uses $N_0-R=1,900$ working samples.

For LPBs, the grid contains zero, a 256-point logarithmic low tail below .001, and 1,000 logarithmically spaced points from .001 to $10^{-.01}$, with prior .56. For UPBs, we use 3,001 equally spaced levels between zero and one, prior .98, and the semantic infinite fallback. We design acquisition separately for 70\%, 80\%, and 90\% coverage. 

\subsection{Choice of \HARP variant}
\label{app:selection}
We compare \HARP-C, \HARP-C+depth, \HARP-D, and \HARP-D+depth using matched calibration/test splits. For each experimental configuration,
we partition 6,000 nontraining prompts into two disjoint cohorts
of 3,000 prompts and perform ten random splits within each
cohort, each with 2,000 calibration prompts and 1,000 test
prompts. Variant selection uses only the ten runs from the
first cohort.

We select one variant per task across all datasets and target
models. For LPBs and UPBs, we rank the variants within each
configuration by their mean squared deviation from the target
coverage, then average these ranks equally across configurations.
Ties are resolved using bound size, favoring larger LPBs and
smaller UPBs. For metric estimation, we apply the same ranking
procedure to the average squared error of the event-rate and
restricted-mean estimates, with restricted-mean error normalized
by the evaluation horizon. This procedure selects \HARP-D+depth for
constructing LPBs and UPBs, and \HARP-D without depth for estimating metrics.
The displayed comparisons include all twenty runs, including
the ten used for selection, and therefore provide descriptive
comparisons rather than an independent held-out evaluation
of the selected variant.

\subsection{Statistics, figures, and provenance}
Each boxplot contains twenty run-level values and shows the median, quartiles, 1.5-IQR whiskers, and all outliers. Variance bars use sample variance with denominator nineteen. We compute ratios within configurations before aggregation; a zero baseline variance gives no finite ratio. Bound comparisons report nominal-coverage deviation, size, and UPB infinity fraction. Metric errors use the full mean of the corresponding evaluation cohort. Raw estimate variance can include differences between cohort means.

Both acquired methods use sequential AHT for event rate and terminal HT for restricted mean. The metric grids show absolute event-rate error, variance, and estimates with the benchmark-cohort means. Uncalibrated uses raw initial predictions and has no acquisition-cost box. We display weighted estimates without clipping them to the target's natural range. Actual spend can differ across methods.

\subsection{Completion and depth ablation}
\label{app:ablation}
We restrict the method ablation to the four C/D and depth variants. Table~\ref{tab:ablation} compares them over all twenty runs, without selecting a winner for each dataset. D uses a closed-form ratio of a model-based risk/cost target to accumulated reach, with a per-turn floor; it does not fit \DAPRO's time/score-bin policy. 

\begin{table}[htbp]\centering
\caption{Four-variant comparison over twenty runs. Entries are median within-configuration variance ratios to \DAPRO: LPB/UPB coverage and the two population metrics. All four variants are retained to assess C/D and depth.}
\label{tab:ablation}
\begin{tabular}{lrrrr}\toprule
Variant & LPB & UPB & Event rate & Restricted mean\\\midrule
\HARP-C (no depth) & 1.935 & 0.622 & 0.632 & 0.466\\
\HARP-C (+depth) & 1.017 & 0.678 & 1.081 & 0.720\\
\HARP-D (no depth) & 1.410 & 0.615 & 0.671 & 0.457\\
\HARP-D (+depth) & 0.948 & 0.585 & 1.200 & 0.819\\
\bottomrule\end{tabular}\end{table}

\subsection{Interpreting the contribution of continuation and depth}
Table~\ref{tab:ablation} shows that depth reduces the median LPB coverage-variance ratio for both C (1.935 to 1.017) and D (1.410 to 0.948). For UPBs, it reduces D's ratio from 0.615 to 0.585 but increases C's from 0.622 to 0.678. D with depth has the smallest median coverage-variance ratio among the four variants for both bound tasks, consistent with its saved selection by coverage-MSE ranks. 

For metrics, depth increases the event-rate and restricted-mean variance ratios under both C and D. In D, the ratios rise from 0.671 to 1.200 for event rate and from 0.457 to 0.819 for restricted mean. The task-level normalized-MSE criterion selects D without depth, whose mean rank is 1.45 versus C's 1.90. C nevertheless has the smaller median event-variance ratio, 0.632 versus 0.671, illustrating why event variance alone does not determine the two-metric choice. The LPB depth objective does not optimize metric MSE, and insured metric depth preserves the target at the cost of larger late-prefix weights. This provides a possible explanation for the observed tradeoff.

\section{All LPB Benchmark Comparisons}
\label{app:lpbexp}
We show each dataset at its sole primary budget, with all four targets and twenty runs. The three rows display coverage and size, absolute coverage deviation and paid cost, and coverage and size variance, respectively. \lpbchoice
\newcommand{\resultpanel}[4]{%
\begin{figure}[p]\centering\paperfigure{#1}{7.0}
\caption{\textbf{#2.} #3}\label{#4}\end{figure}}
\resultpanel{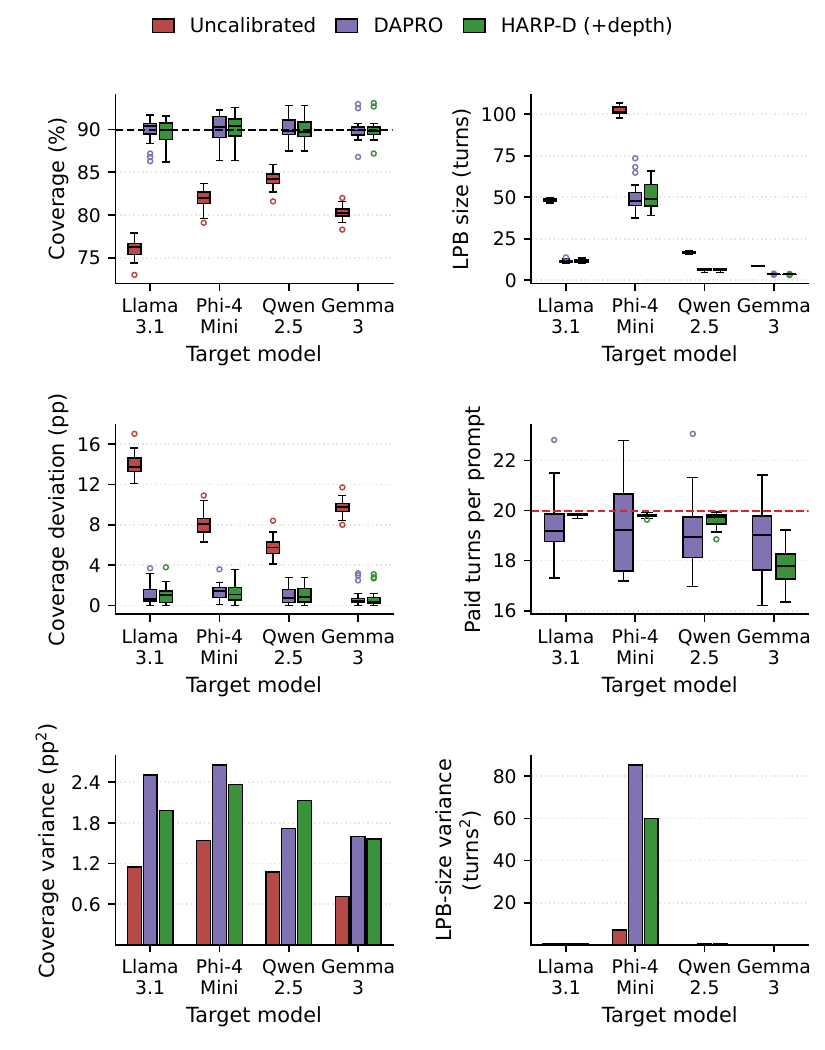}{Toxicity LPBs}{90\% coverage, $\bar B=20$ turns per sample, twenty runs per target. \lpbchoice\ Black/red dashed lines denote coverage/budget. All methods use matched splits.}{fig:lpb-toxicity}
\resultpanel{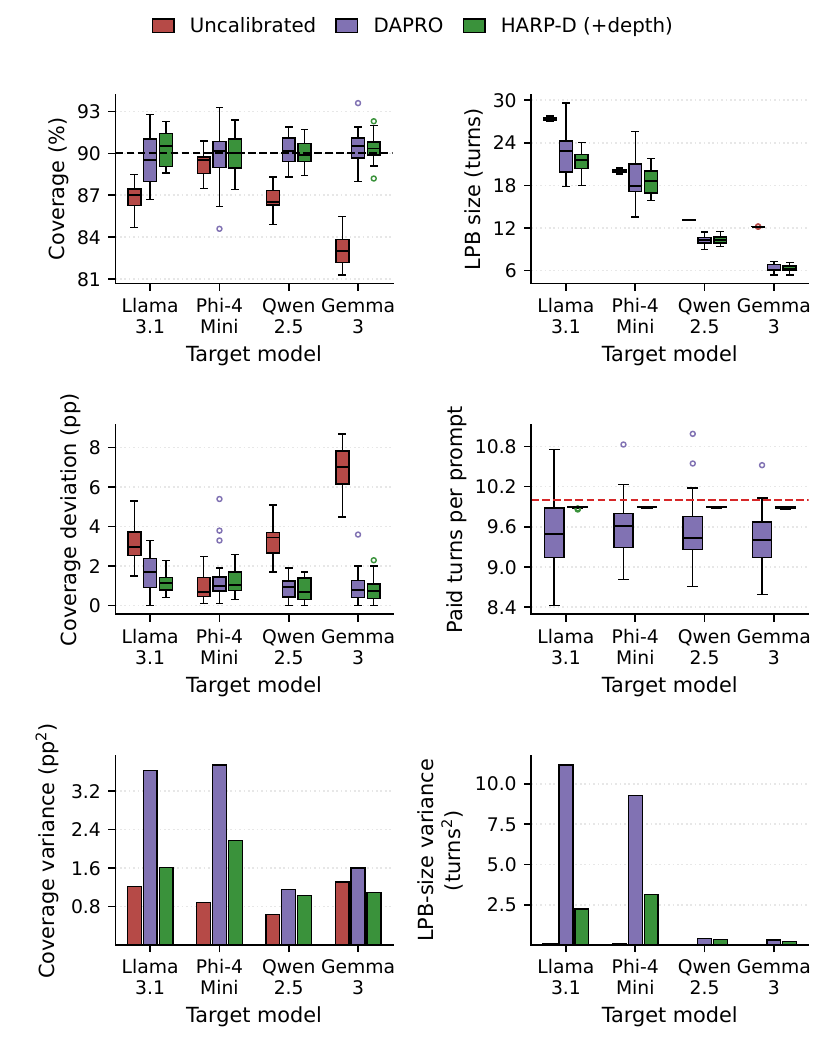}{Red team, Qwen judge LPBs}{90\% coverage, $\bar B=10$ turns per sample, twenty runs per target. \lpbchoice\ Black/red dashed lines denote coverage/budget. All methods use matched splits.}{fig:lpb-red_team_qwen}
\resultpanel{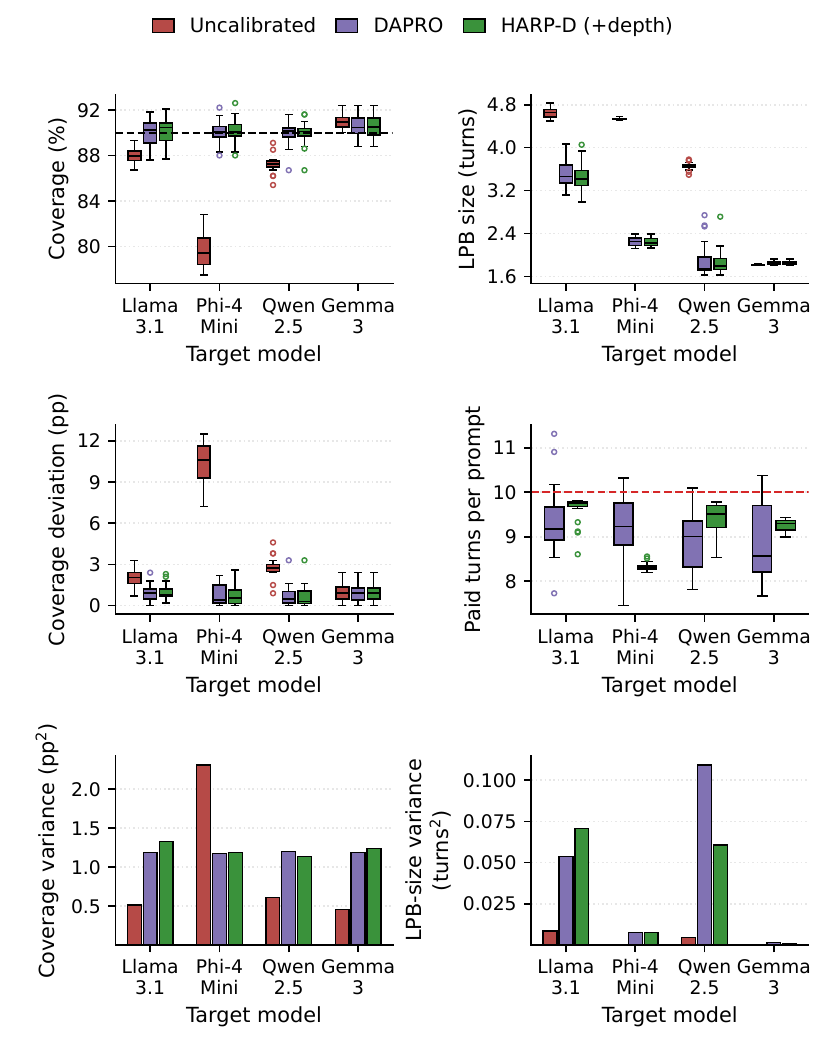}{Red team, Llama Guard LPBs}{90\% coverage, $\bar B=10$ turns per sample, twenty runs per target. \lpbchoice\ Black/red dashed lines denote coverage/budget. All methods use matched splits.}{fig:lpb-red_team_llama_guard}
\resultpanel{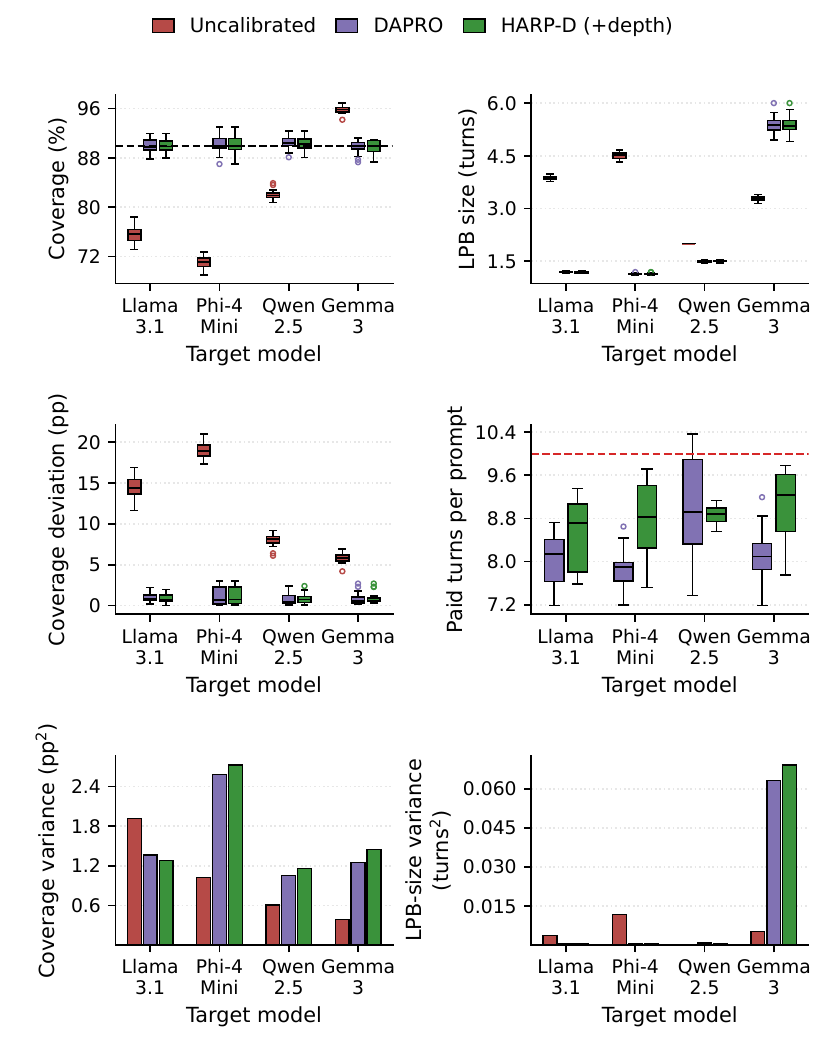}{Hallucination LPBs}{90\% coverage, $\bar B=10$ turns per sample, twenty runs per target. \lpbchoice\ Black/red dashed lines denote coverage/budget. All methods use matched splits.}{fig:lpb-hallucination3}
\resultpanel{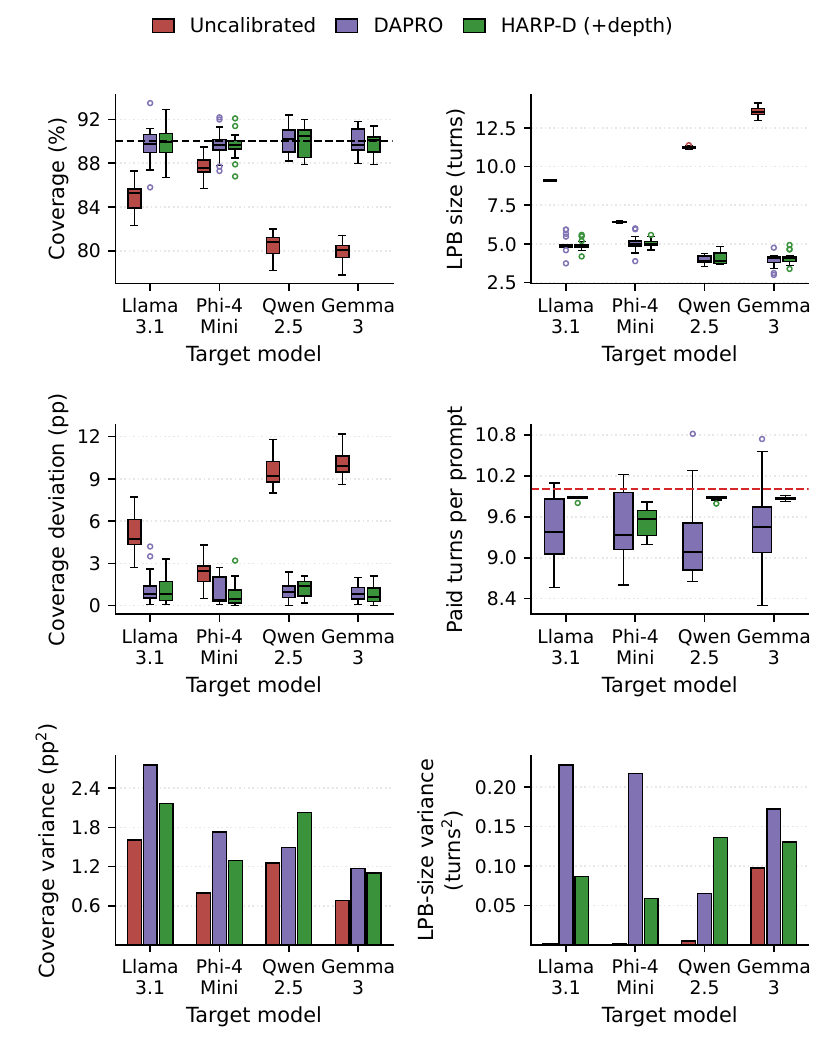}{AutoIF LPBs}{90\% coverage, $\bar B=10$ turns per sample, twenty runs per target. \lpbchoice\ Black/red dashed lines denote coverage/budget. All methods use matched splits.}{fig:lpb-autoif}
\clearpage
\section{All Population-Metric Benchmark Comparisons}
\label{app:metricexp}
For each dataset, we show one grid at $\bar B=20$, with all four targets and twenty runs. The rows display event-rate estimates and variance, restricted-mean estimates and variance, and paid cost and absolute event-rate error. Red dashed marks show the benchmark-cohort means; metric errors use the mean corresponding to each run. Both acquired methods use the same estimators. \metricchoice
\resultpanel{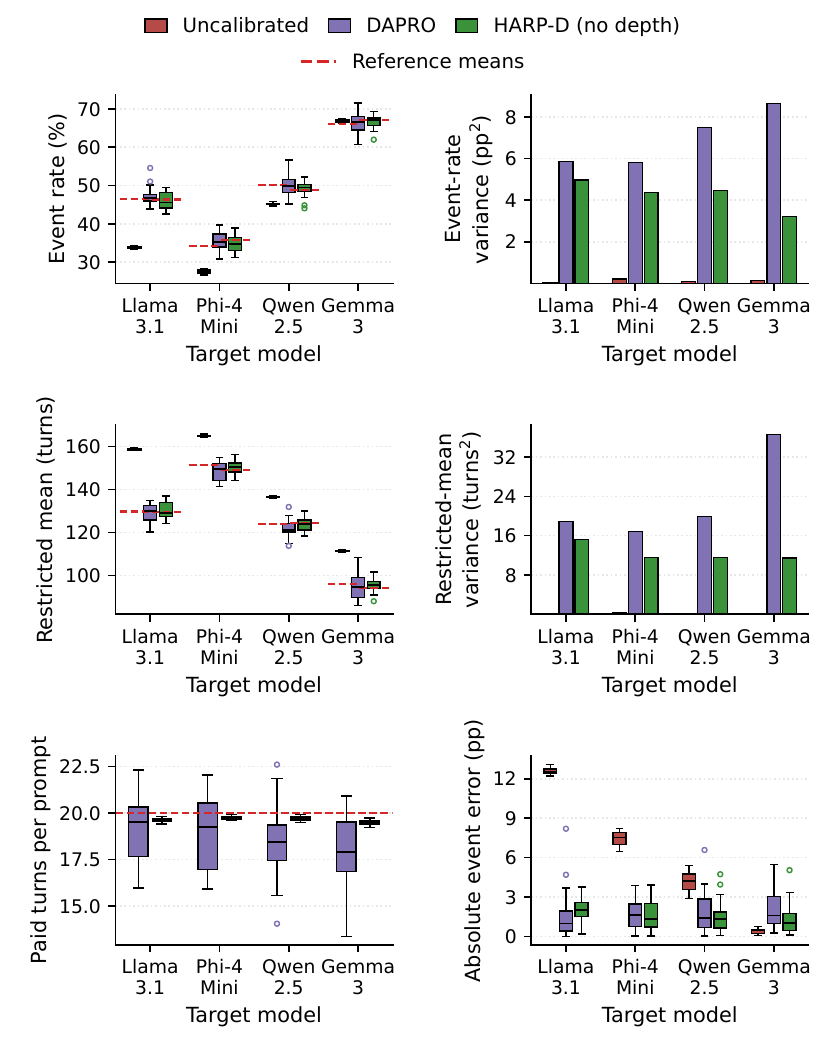}{Toxicity population metrics}{Budget $\bar B=20$ turns per sample, twenty runs per target. \metricchoice\ The event is harmful/ungrounded in safety settings and successful completion for AutoIF; restricted mean is $\E[\min(T,200)]$. Red cost lines mark the cap.}{fig:metric-toxicity}
\resultpanel{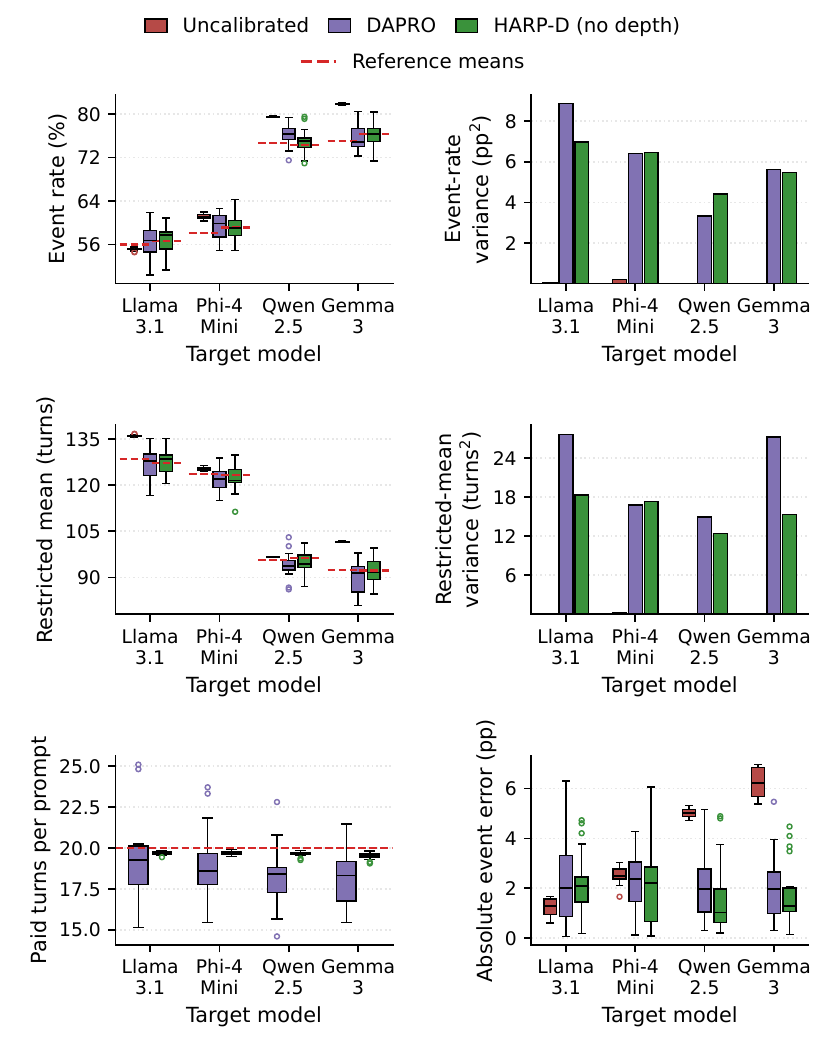}{Red team, Qwen judge population metrics}{Budget $\bar B=20$ turns per sample, twenty runs per target. \metricchoice\ The event is harmful/ungrounded in safety settings and successful completion for AutoIF; restricted mean is $\E[\min(T,200)]$. Red cost lines mark the cap.}{fig:metric-red_team_qwen}
\resultpanel{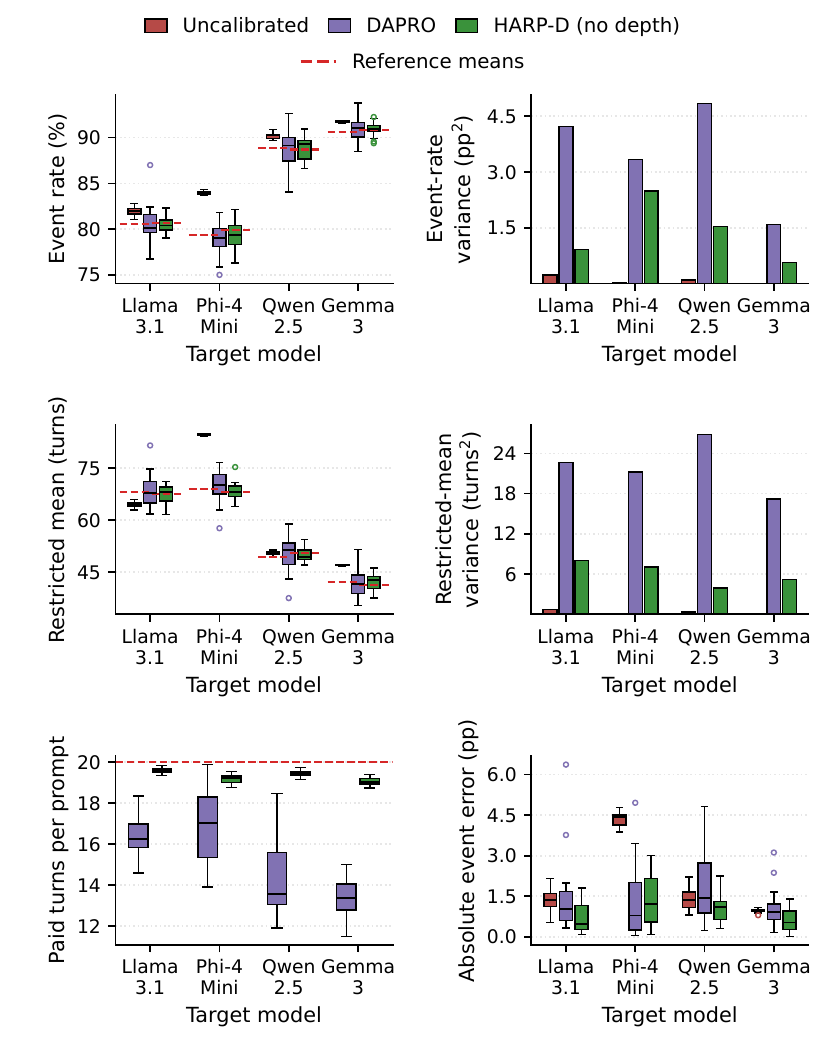}{Red team, Llama Guard population metrics}{Budget $\bar B=20$ turns per sample, twenty runs per target. \metricchoice\ The event is harmful/ungrounded in safety settings and successful completion for AutoIF; restricted mean is $\E[\min(T,200)]$. Red cost lines mark the cap.}{fig:metric-red_team_llama_guard}
\resultpanel{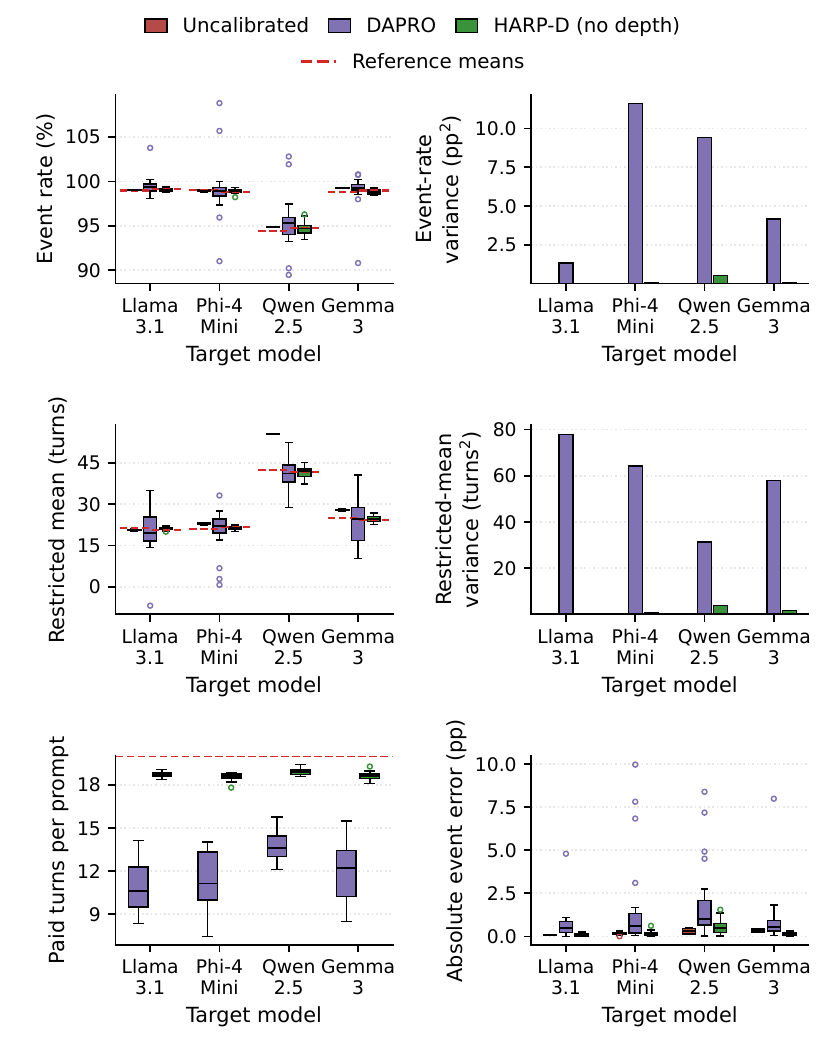}{Hallucination population metrics}{Budget $\bar B=20$ turns per sample, twenty runs per target. \metricchoice\ The event is harmful/ungrounded in safety settings and successful completion for AutoIF; restricted mean is $\E[\min(T,200)]$. Red cost lines mark the cap.}{fig:metric-hallucination3}
\resultpanel{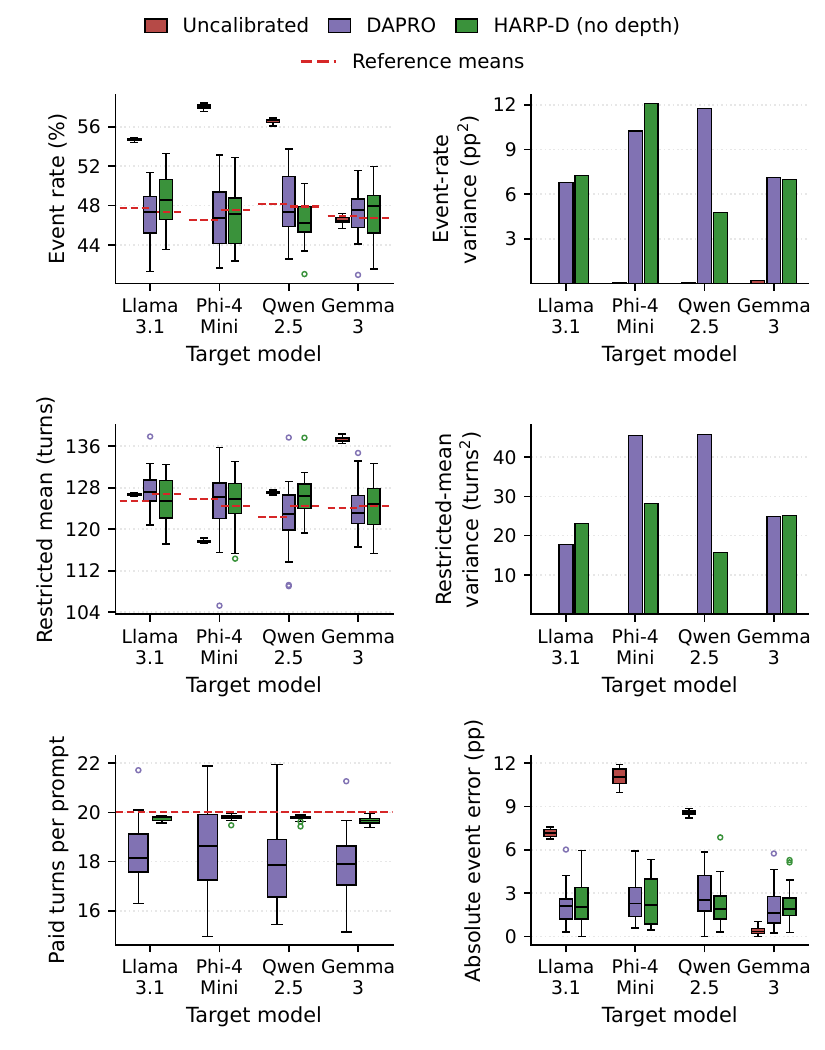}{AutoIF population metrics}{Budget $\bar B=20$ turns per sample, twenty runs per target. \metricchoice\ The event is harmful/ungrounded in safety settings and successful completion for AutoIF; restricted mean is $\E[\min(T,200)]$. Red cost lines mark the cap.}{fig:metric-autoif}
\end{document}